\documentclass[twoside,11pt]{article}

\usepackage{blindtext}

\usepackage[preprint]{jmlr2e}

\usepackage{times}
\usepackage[utf8]{inputenc}
\usepackage{babel}
\usepackage{upgreek}
\usepackage{enumitem}
\usepackage{mathrsfs}
\usepackage{bbm}
\usepackage{amssymb}
\usepackage{amsmath}
\usepackage{multirow}
\usepackage{bbold}
\usepackage{graphicx}
\graphicspath{ {./} }
\usepackage{mathtools}
\usepackage{algorithm2e}
\usepackage{xcolor}
\usepackage{placeins}
\usepackage{caption}
\usepackage{thmtools,thm-restate}
\usepackage{csquotes}
\usepackage{appendix}
\usepackage{dsfont}
\usepackage{newfloat}

\newcommand{\mcA}{\mathcal{A}}
\newcommand{\mcB}{\mathcal{B}}
\newcommand{\mcC}{\mathcal{C}}
\newcommand{\mcD}{\mathcal{D}}
\newcommand{\mcE}{\mathcal{E}}

\newcommand{\mcG}{\mathcal{G}}

\newcommand{\mcI}{\mathcal{I}}
\newcommand{\mcJ}{\mathcal{J}}

\newcommand{\mcM}{\mathcal{M}}

\newcommand{\mcS}{\mathcal{S}}

\newcommand{\mcX}{\mathcal{X}}

\newcommand{\bigO}[1]{\mathcal{O}\left(#1\right)}

\newcommand{\supp}[1]{\mathbf{supp}\left(#1\right)}

\newcommand{\HSig}{{\widehat{\Sigma}}}

\newcommand{\Sig}{\Sigma}
\newcommand{\hmcS}{{\widehat{\mathcal{S}}}}

\newcommand{\bfx}{\mathbf{x}}

\newcommand{\bfw}{\mathbf{w}}
\newcommand{\bfN}{\mathbf{N}}
\newcommand{\bfzo}{\mathbf{0}}
\newcommand{\bfA}{\mathbf{A}}

\newcommand{\bfV}{\mathbf{V}}
\newcommand{\tbfV}{\widetilde{\bfV}}
\newcommand{\pbfV}{\tbfV^{\text{proj}}}
\newcommand{\bfW}{\mathbf{W}}

\newcommand{\bfX}{\mathbf{X}}
\newcommand{\bfY}{\mathbf{Y}}
\newcommand{\bfz}{\mathbf{z}}
\newcommand{\bfD}{\mathbf{D}}
\newcommand{\hbfD}{\widehat{\mathbf{D}}}
\newcommand{\tbfD}{\widetilde{\mathbf{D}}}
\newcommand{\bbfD}{\bar{\bfD}}
\newcommand{\bfB}{\mathbf{B}}

\newcommand{\bfI}{\mathbf{I}}
\newcommand{\bfP}{\mathbf{P}}

\newcommand{\bfF}{\mathbf{F}}

\newcommand{\bfE}{\mathbf{E}}

\newcommand{\bfu}{\mathbf{u}}

\newcommand{\bfe}{\mathbf{e}}

\newcommand{\bfy}{\mathbf{y}}

\newcommand{\bfd}{\mathbf{d}}
\newcommand{\hbfd}{\hat{\mathbf{d}}}
\newcommand{\tbfd}{\tilde{\mathbf{d}}}
\newcommand{\bbfd}{\bar{\bfd}}

\newcommand{\bfv}{\mathbf{v}}
\newcommand{\bbi}{\mathds{1}}

\newcommand{\tbfy}{{\tilde{\bfy}}}

\newcommand{\epsi}{\varepsilon}

\newcommand{\R}{\mathbb{R}}

\newcommand{\BBP}[1]{{\mathbb{P}}\left(#1\right)}

\newcommand{\inner}[2]{\left\langle #1, #2 \right\rangle}
\newcommand{\sign}[1]{\mathbf{sgn}\left(#1\right)}

\newcommand{\vspan}{\mathbf{span}}
\newcommand{\proj}[2]{{\mathbf{proj}_{#1}(#2)}}

\newcommand{\OZ}{{\Omega_0}}

\newcommand{\tA}{\widetilde{A}}
\newcommand{\tO}{\widetilde{\Omega}}

\declaretheorem[name=Theorem,numberwithin=section]{thm}
\declaretheorem[name=Lemma,numberwithin=section]{lma}

 \hypersetup{
    hidelinks
     }

\usepackage{lastpage}
\jmlrheading{23}{2023}{1-\pageref{LastPage}}{3/27; Revised TBD}{TBD}{21-0000}{Alexei Novikov and Stephen White}

\ShortHeadings{Overcomplete Dictionary Learning for the Almost-Linear Sparsity Regime}{Novikov and White}
\firstpageno{1}

\begin{document}

\title{Dictionary Learning for the Almost-Linear Sparsity Regime}

\author{\name Alexei Novikov \email novikov@psu.edu \\
       \addr Department of Mathematics\\
       Penn State University\\
       University Park, PA 16802 USA
       \AND
       \name Stephen White \email sew347@psu.edu \\
       \addr Department of Mathematics\\
       Penn State University\\
       University Park, PA 16802 USA}

\editor{Francis Bach and David Blei}

\maketitle

\begin{abstract}
\textit{Dictionary learning}, the problem of recovering a sparsely used matrix $\mathbf{D} \in \mathbb{R}^{M \times K}$ and $N$ $s$-sparse vectors $\mathbf{x}_i \in \mathbb{R}^{K}$ from samples of the form $\mathbf{y}_i = \mathbf{D}\mathbf{x}_i$, is of increasing importance to applications in signal processing and data science. When the dictionary is known, recovery of $\mathbf{x}_i$ is possible even for sparsity linear in dimension $M$, yet to date, the only algorithms which provably succeed in the linear sparsity regime are Riemannian trust-region methods, which are limited to orthogonal dictionaries, and methods based on the sum-of-squares hierarchy, which requires super-polynomial time in order to obtain an error which decays in $M$. In this work, we introduce \textbf{SPORADIC} (\textbf{Sp}ectral \textbf{Ora}cle \textbf{Dic}tionary Learning), an efficient spectral method on family of reweighted covariance matrices. We prove that in high enough dimensions, \textbf{SPORADIC} can recover overcomplete ($K > M$) dictionaries satisfying the well-known restricted isometry property (RIP) even when sparsity is linear in dimension up to logarithmic factors. Moreover, these accuracy guarantees have an ``oracle property" that the support and signs of the unknown sparse vectors $\mathbf{x}_i$ can be recovered exactly with high probability, allowing for arbitrarily close estimation of $\mathbf{D}$ with enough samples in polynomial time.  To the author's knowledge, \textbf{SPORADIC} is the first polynomial-time algorithm which provably enjoys such convergence guarantees for overcomplete RIP matrices in the near-linear sparsity regime.
\end{abstract}


\section{Introduction}

In data science and machine learning, the task of discovering sparse representations for large datasets is of central importance. Sparse representations offer distinct advantages for storing and processing data, as well as providing valuable insights into the underlying structure of a dataset. This problem of sparse recovery frequently involves the recovery of a sparse vector $\bfx \in \R^K$ from a dense sample $\bfy = \bfD\bfx$, where D is a sparsely-used matrix referred to as the "dictionary.'' Depending on applications, this dictionary can be known from physics or hand-designed, such as the wavelet bases used in image processing \citep[e.g.,][]{wavelets}. Beyond numerous further applications in signal and image processing (see \citet{sparse_img_book} for a summary of developments), sparse representations have been fruitfully applied in areas including computational neuroscience \citep{neuro_intro_1, neuro_intro_2, neuro_intro_3} and machine learning \citep{ML_intro_1, ML_intro_2}.

In today's technology-driven world, interpreting and compressing increasingly varied types of data have become paramount tasks. This has sparked a growing interest in techniques that not only learn the fundamental sparse codes from data but also the dictionary itself: this is the \textit{dictionary learning problem}. Specifically, we aim to recover a sparsely-used matrix $\bfD \in \R^{M \times K}$ from $N$ samples of the form $\bfy = \bfD\bfx$, where $\bfx \in \R^K$ is sufficiently sparse (that is, $\bfx$ has significantly fewer than $K$ nonzero entries). Written in matrix form, we seek to recover $\bfD$ from a matrix $\bfY = \bfD\bfX \in \R^{M\times N}$ with the prior knowledge that rows of $\bfX$ are sparse. In many applications, practitioners are particularly interested in recovering \textit{overcomplete} dictionaries which satisfy $K > M$, as such dictionaries offer greater flexibility in selecting a basis and allow for a sparser representation \citep[e.g.][]{ChenOvercomplete, DonohoOvercomplete}.

\begin{figure}
\centering
\includegraphics[width=0.7\linewidth]{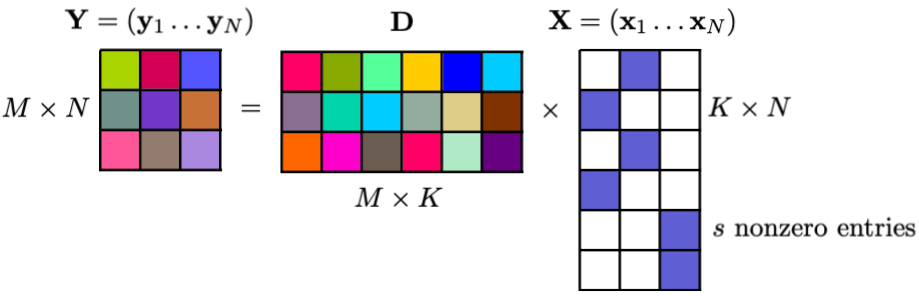}
\caption{A sample $\bfY$ as encountered in dictionary learning. $\bfD$ is an unknown and typically dense matrix, and each $\bfx_i$ is an unknown vector with at most $s$ nonzero entries.}
\end{figure}

\subsection{Prior Work}

In computer science literature, dictionary learning is typically formulated as a nonconvex optimization problem of finding a dictionary $\bfD$ and matrix with sparse columns $\bfX$ satisfying $\bfY = \bfD\bfX$ such that $\bfX$ is as sparse as possible, namely:
\[
\text{Find } \bfD, \bfX \text{ minimizing } \|\bfX\|_0 \text{ subject to } \bfY = \bfD\bfX \text{.}
\]
As a nonconvex optimization, this problem is computationally challenging to solve. The most popular heuristic for solving the dictionary learning problem is \textit{alternating minimization}. Alternating minimization algorithms rely on the fact that when $\bfD$ or $\bfX$ is known, the other can be solved using known methods, most frequently based on $L^1$ relaxations that make the problem convex \citep{CandesCS}. Alternating minimization techniques thus alternate between a ``sparse coding'' step in which a guess for $\bfD$ is fixed and the algorithm solves for $\bfX$, and a ``dictionary update'' step in which $\bfX$ is fixed and the dictionary is updated. This process is repeated until a convergence criterion is met. Though theoretical guarantees for convergence of such methods exists, they often demand extremely sparse representations and precise initialization \citep{CH_Alternating, am1}.

In this paper, we are interested in dictionary learning algorithms with provable guarantees. Initial theoretical study of provable dictionary learning focused on the case when $s$ is no greater than $\sqrt{M}$; this is a well-known recovery boundary even when the dictionary $\bfD$ is known. \citet{Spielman} developed an algorithm that accurately recovers the dictionary in this sparsity regime, but their algorithm does not generalize to recovery of overcomplete dictionaries. \citet{Arora13} and \citet{Agarwal13} then independently introduced similar correlation-and-clustering methods, which enjoy similar theoretical guarantees for the $s \sim \sqrt{M}$ regime for overcomplete dictionaries. 

\citet{CandesCS} showed that when the dictionary $\bfD$ is known and satisfies a restricted isometry property (RIP; see Definition \ref{RIP}) \citep[][]{RIP}, it is possible efficiently to recover $\bfx_i$ from $\bfy_i = \bfD\bfx_i$ when $\bfx_i$ has linearly-many (in $M$) nonzero entries. Accordingly, there was tremendous interest in determining whether recovery in this scaling regime remained possible when $\bfD$ is unknown. In \citep{Arora14}, the authors develop provable methods for recovering dictionaries with sparsity linear in $M$ up to logarithmic factors, but their method requires quasipolynomial running time. In a pair of papers, Sun et al. develop a polynomial-time method which can provably recover invertible dictionaries with $s = \bigO{M}$ \citep{Sun1, Sun2}. However, this algorithm depends intimately on properties of orthogonal and invertible matrices and thus is limited to the case of complete dictionaries ($M = K$); moreover, their theoretical guarantees demand a very high sample complexity of $K \gg M^8$. More recently, \citet{ZhaiL4} introduced a method based on $L^4$ norm optimization. Despite many nice properties of this approach further elaborated by \citet{ZhaiUnderstanding}, it remains limited to the complete case and the authors prove only the accuracy of a global optimum without a guarantee of convergence in arbitrary dimension.

In the overcomplete setting, \citet{Barak} developed a tensor decomposition method based on the sum-of-squares hierarchy that can recover overcomplete dictionaries with sparsity up to $s = \bigO{M^{1-\gamma}}$ for any $\gamma>0$ in polynomial time, but this time tends to super-polynomial as $\gamma \to 0$ and requires a constant error as $M \to \infty$. This and related methods have generally enjoyed the best theoretical guarantees for efficient dictionary learning in the overcomplete linear sparsity regime due to their impressive generality, especially after their runtime was improved to polynomial time by \citet{tensor_polynomial} provided that the target error remains constant. Yet the requirement of constant error is strict---with these methods, even in sublinear sparsity regimes such as $s \sim M^{0.99}$, an inverse logarithmic decay in error requires superpolynomial time.

\subsection{Our Contribution}
 
For \textit{known} overcomplete RIP matrices $\bfD$, $\bfX$ can be recovered from $\bfY$ for sparsity linear in $M$ up to log factors by methods such as $L^1$-minimization \citep{RIP} or lasso \citep{lasso}. While results such as those of \citet{Barak} provide very general guarantees for wide families of matrices in this regime, these methods are slow owing to sum-of-squares techniques requiring multiple levels of convex relaxation. Until now, there are no results extending dictionary learning specifically to RIP matrices, where we might expect much better recovery properties owing to the success of such methods for proving guarantees in the context of compressed sensing. 

In this work, we answer this question in the affirmative, showing that matrices with good RIP constants along with a ``uniform variance'' assumption can be recovered efficiently even with sparsity linear in $M$ up to logarithmic factors. These results are global: we require no initialization for recovery to work; moreover, our guarantees have the ``oracle property'' that both the support and signs of the sparsity pattern matrix $\bfX$ are recovered exactly with high probability. This allows for arbitrarily close recovery given enough samples, and even for truly exact recovery of $\bfD$ if one is willing to make stronger assumptions on the distribution of $\bfX$. This is a significant improvement over the best known alternatives for the overcomplete near-linear sparsity regime, which only achieved constant \citep{tensor_polynomial} error bounds in polynomial time. In addition to the benefits outlined above, our proposed algorithm \textbf{SPORADIC} (\textbf{SP}arse \textbf{ORA}cle \textbf{DIC}tionary Learning) is conceptually simple, easy to implement, and easily parallelized.



\subsection{Intuition}

The correlation-based clustering method of \citet{Arora14} offers an appealing intuition in the $s \sim \sqrt{M}$ regime: that pairs of samples $\bfy_i,\bfy_j$ which ``look similar" in the sense of being highly correlated ($|\inner{\bfy_i}{\bfy_j}|$ is in some sense large) are likely to share support. More technically, if the dictionary $\bfD$ is incoherent (that is, $|\inner{\bfd_{k_1}}{\bfd_{k_2}}| \leq c/\sqrt{M}$ for $k_1 \neq k_2$) and the coefficients $\bfx_i,\bfx_j$ are symmetric, then if $s \ll \sqrt{M}$ then as $M \to \infty$, the \textit{correlation}\footnote{We note this is a slight abuse of notation as the vectors $\bfy_i, \bfy_j$ may not be unit vectors.} $\inner{\bfy_i}{\bfy_j}$ will concentrate near zero if $\bfy_i$ and $\bfy_j$ share no support, but concentrate around $\pm 1$ if they do. As a result, thresholding $|\inner{\bfy_i}{\bfy_j}|$ becomes a reliable indicator of whether $\bfy_i$ and $\bfy_j$ share a common dictionary element in their support. 

By constructing a graph with an edge between $i$ and $j$ whenever $|\inner{\bfy_i}{\bfy_j}| \geq \uptau$ for some threshold $\uptau$ (say 1/2), one can determine which groups of $\bfy_i$'s share \textit{the same} common support element by applying overlapping clustering methods. Yet correlation-based clustering cannot be performed with accuracy once sparsity exceeds $\sqrt{M}$, as above this threshold correlation no longer reliably indicates whether two samples share a common dictionary element. As a result, methods based on the correlation $\inner{\bfy_i}{\bfy_j}$ have not been widely employed in subsequent attempts to solve dictionary learning in the linear sparsity regime.


We adopt a different approach which sidesteps the key challenges of these previous correlation methods, allowing us to apply these methods successfully even in the linear sparsity regime.  Our proposed algorithm \textbf{SPORADIC} works by a natural adaptation of the correlation-based approach of \citet{Arora13} to the linear sparsity regime. Instead of immediately attempting to recover dictionary elements, we first pursue an intermediate step of recovering \textit{spanning subspaces}, the subspaces $\mcS_i$ spanned by the supporting dictionary elements of each sample $\bfy_i$. Once these subspaces are recovered, the individual dictionary elements can be recovered through pairwise comparison of subspaces to find their intersection. 


To recover these subspaces, we employ a spectral method based on extracting the eigenvectors of a modified covariance matrix of the samples $\bfY$. Specifically, for each $j$ we examine the \textit{correlation-weighted covariance matrix} $\HSig_j$, defined as the sample covariance of the correlation-weighted samples $\inner{\bfy_i}{\bfy_j}\bfy_i$. By design, these reweighted samples will have greater variance in the directions of the support elements of $\bfy_j$, meaning $\HSig_j$ will have a rank-$s$ ``spike'' in the directions spanned by support elements of $\bfy_j$, while a random-matrix assumption on $\bfD$ guarantees that, with high probability, there will be no comparable spikes in other directions. As a result, the $s$ leading eigenvectors of $\HSig_j$ (that is, the $s$ eigenvectors corresponding to the $s$ largest eigenvalues) will reliably span a subspace close to that spanned by the support elements of $\bfy_j$. We provide theoretical guarantees that this method accurately recovers spanning subspaces with sparsity $s \sim M\log^{-(6+\eta)} (M)$ for any $\eta > 0$, which allows for recovery of the individual dictionary elements by a further subspace intersection process (see Algorithm \ref{subspace_int_approx}). 

Although the resulting estimate enjoys only weak convergence guarantees---$\bigO{1/\text{polylog}(M)}$---it is accurate enough that subsequent post-processing steps can improve the estimate drastically. After obtaining a nontrivial estimate of the dictionary which is close column-wise, we show how these estimated columns can be used as ``approximate oracles'' to find a refined dictionary estimate with substantially better error bounds. These bounds are tight enough to allow recovery of the support \textit{and} signs of $\bfX$, from which an even better estimate can be gained by simple averaging. Most significantly, the error bounds for this final estimate tend to zero in the number of samples $N$, and truly exact recovery is even possible in certain cases.

The sample complexity required for \textbf{SPORADIC} to recover subspaces accurately depends on the particular sparsity regime. In the most challenging linear-sparsity regime that is our main focus, up to $\log$ factors subspace recovery requires a sample complexity of at most $M^4$ for the initial dictionary estimate to be accurate, but in the ``easier" regime $s \approx \sqrt{M}$, our results suggest the required sample complexity eases to $N \sim M$ (see Theorem \ref{master_thm}). In the linear-sparsity case, the bottleneck is caused by approximation of a covariance matrix in Frobenius norm, which is known to require a factor of $M$ additional samples than does estimation in the operator norm. We believe that in future work this step can be replaced by an approach requiring approximation only in the $L^2$ operator norm, in which case the worst-case sample complexity would be lowered to $M^3$.



\section{Conventions and Data Model}\label{conventions}
We begin by stating the sparse dictionary learning problem explicitly:

\begin{definition}[Sparse Dictionary Learning]
Let $\bfD = \begin{pmatrix} \bfd_1 & \bfd_2 & \ldots & \bfd_K \end{pmatrix}$ be an (unknown) $M \times K$ matrix with unit vector columns, called the \emph{dictionary}. Let $\bfx$ be an $s$-sparse random vector, and define the random vector $\bfy = \bfD\bfx$. The \textit{sparse dictionary learning problem} is:
\[
\emph{Given } \bfY = \bfD\bfX \emph{ where } \bfX \emph{ is a } K \times N \emph{ matrix with columns } \{\bfx_i\}_{i=1}^N \emph{ i.i.d. copies of } \bfx \emph{, recover } \bfD \emph{.}
\]
\end{definition}

It is clear from the definition that $\bfD$ can only be recovered up to sign and permutation. Accordingly, we employ the following definition for comparing two dictionaries, due to \citet{Arora13}: we say that two dictionaries are \textit{column-wise $\epsi$-close} if their columns are close in Euclidean norm after an appropriate permutation and change of sign. In detail:

\begin{definition}[Column-wise $\epsi$-close \citep{Arora13}]
Two dictionaries $\bfD = \begin{pmatrix} \bfd_1 & \bfd_2 & \ldots & \bfd_{K}\end{pmatrix}$ and $\bfD' = \begin{pmatrix} \bfd'_1 & \bfd'_2 & \ldots & \bfd'_{K'}\end{pmatrix}$ are \textit{column-wise $\epsi$-close} if they have the same dimensions $M \times K$ and there exists a permutation $\pi$ of $\{1,\ldots,K\}$ and a $K$-element sequence $\theta_k \in \{-1,1\}$ such that for all $k = 1, \ldots, K$:
\[
\|\bfd_k - \theta_k\bfd'_{\pi(k)}\|_2 \leq \varepsilon
\]
\end{definition}

\subsection{Notation and conventions}
Vectors are represented by boldface lowercase letters, while matrices will be written as boldface uppercase letters. Roman letters (both upper- and lowercase) will be used for both scalars and random variables depending on context. We will use the notation $|\mcA|$ for the number of elements in a finite set $\mcA$, and $\mcA^c$ for its complement.

We use two matrix norms at different points in the text. The standard $L^2$ operator norm will be denoted by $\|\bullet\|_2$ while the Frobenius norm will be denoted $\|\bullet\|_F$. Vector norms always refer to the standard $L^2$ (Euclidean) norm, and will be denoted $\|\bullet\|_2$. We will use the notation $a \ll b$, where both $a$ and $b$ are scalars depending on $M$, to mean $\lim_{M\to\infty}|a|/|b| = 0$, where the norm in question may depend on context.

The index-free notation $\bfy = \bfD\bfx$ will refer to a generic independent copy drawn from the sampling distribution, used for index-independent properties of this distribution such as expectation, while we reserve the indexed notation $\bfy_i$ to refer to a particular random vector in the sample $\bfY$. Given a sample $\bfy_i$, its \textit{support}, denoted $\Omega_i$, is defined as the set of indices of the dictionary vectors in its construction with nonzero coefficients:
\[
\Omega_i := \supp{\bfx_i} = \{k \in \{1,\ldots,K\}: x_{ik} \neq 0\}
\]
The ``support vectors'' of $\bfy_i$ refer to the dictionary elements indexed by $\Omega_i$, the set $\{\bfd_k\}_{k\in \Omega_i}$. We use the notation $\mcA - \mcB$ for the relative complement of set $\mcB$ in set $\mcA$: $\mcA-\mcB = \{x: x \in \mcA, x \notin \mcB\}$. We denote the dimension of a vector subspace $\mcS$ with the shorthand $\dim(\mcS)$. We use $\bbi$ for the indicator function; for example, $\bbi_{k \in \Omega_i}$ equals one for $k \in \Omega_i$ and zero otherwise.

\subsection{Data Model}\label{data_model}

We begin by defining the following distribution for the sparsity pattern $\bfX$:


\begin{definition}[$\mcX(W)$ distribution]
Let $W$ be a symmetric random variable satisfying $|W| \in [c,C]$ almost surely for $0 < c \leq C$. A random vector $\bfX \in \R^{K \times N}$ follows a $\mcX(W)$ distribution if:
\begin{itemize}
    \item The supports $\Omega_i = \supp{\bfx_i}$ of each column $\bfx_i$ of $\bfX$ are independent, uniformly random $s$-element subsets of $\{1,\ldots,K\}$.
    \item Nonzero entries of $\bfX$ are i.i.d. copies of $W$.
\end{itemize}
\end{definition}

This definition implies columns of $\bfX$ are independent when distributed according to a $\mcX(W)$ distribution. In our theoretical results, we will assume that $\bfX \sim \mcX(W)$. As the extension to bounded symmetric random variables is trivial, we will prove our results for $W=\pm 1$ with equal probability. This choice of particular distribution for $\bfX$ is made for theoretical convenience and significantly simplifies the analysis, but we expect our results to hold with minimal modifications for the more commonly used Bernoulli-Gaussian model used by \citet{Spielman} and others. The requirement that nonzero entries $\bfX$ is symmetric is likely essential as it causes many terms to vanish in expectation; however, with i.i.d. one can always ensure a symmetric coefficient distribution by subtracting adjacent samples, i.e. $\bfy_1' = \bfy_1 - \bfy_2$, which will increase sparsity only by a factor of 2.

\subsubsection{Restricted Isometry Property and Recovery Conditions}

To describe our requirements on the dictionary $\bfD$, we now formally introduce the restricted isometry property, initially due to \cite{RIP}:

\begin{definition}[Restricted Isometry Constants]\label{RIP}
Let $\bfD$ be a $M \times K$ matrix. For every $s \in \{1,\ldots,K\}$ we define the $s$-restricted isometry constant of $\bfD$ as the smallest number $\delta_s$ such that
\[
(1-\delta_s)\|\bfx\|_2 \leq \|\bfD\bfx\|_2 \leq (1+\delta_s)\|\bfx\|_2
\]
for every vector $\bfx \in \R^K$ with at most $s$ nonzero entries.
\end{definition}

The following lemma is a consequence of the definition:
\begin{lma}[Restricted orthogonality] \label{lemma:rop}
Let $\bfx$ and $\bfy$ be at-most-$s$-sparse vectors with $\supp{\bfx} \cap \supp{\bfy} = \emptyset$. Then
\[
\inner{\bfD\bfx}{\bfD\bfy} \leq \delta_{2s}\|\bfx\|_2\|\bfy\|_2.
\]
\end{lma}
This is sometimes called the restricted orthogonality property.

To prove recovery, we require two conditions on the dictionary $\bfD$: first, that $\bfD$ is uniform in the sense that it has ``covariance'' $\bfD\bfD^T$ close to the identity, and second that the dictionary has a good $2s$-RIP constant. We also require $K/M^{3/2} \leq \sqrt{K}/\sqrt{M}$ for technical convenience; this will be need to be true for recovery in the almost-linear sparsity regime of interest anyway.

\begin{definition}[Good dictionaries]\label{good_event}
Let $s < M < K$ with $K/M^{3/2} \leq \sqrt{K}/\sqrt{M}$. Then for any constant $B$, we denote by $\mcG(B)$ the class of good dictionaries, defined as the set of dictionaries $\bfD \in \R^{M\times K}$ such that:
\begin{enumerate}[label=$\mcG$.\arabic*]
    \item \label{G1:DDT_norm} $\left\|\bfD\bfD^T - \frac{K}{M}\bfI\right\|_2 \leq \frac{B\sqrt{K}}{\sqrt{M}}$.
    \item \label{G2:1_rip} For all $k \neq m$, $|\inner{\bfd_k}{\bfd_m}| \leq C\log M/\sqrt{M}$.
    \item \label{G2:s_rip} $\bfD$ has $2s$-restricted isometry constant satisfying $\delta_{2s} \leq \frac{B\sqrt{s}\log M}{\sqrt{M}} < \frac{1}{8}$
\end{enumerate}
\end{definition}

The following lemma proves existence of a reasonable class of dictionaries in $\mcG(B)$; proof is in in Appendix \ref{appdx:g0}:
\begin{restatable}{lma}{Glemma}\label{G0_lemma}
Let $\bfD \in \R^{M\times K}$ be a dictionary with columns $\bfd_k$ chosen independently and uniformly from the unit sphere in $\R^M$. Then for large enough $M$, there exists an absolute constant $B$ such that $\bfD \in \mcG(B)$ with high probability; that is, that for any $a>0$,
\[
\lim_{M \to \infty} M^a\BBP{\bfD \notin \mcG(B)} = 0
\]
\end{restatable}

The constant $B$ is not important in an our analysis and so will be wrapped in with the general absolute constants $c$ and $C$ in our theoretical guarantees. Therefore in the following we will simply refer to $\mcG(B)$ as $\mcG$ with $B$ implied as some sufficiently large absolute constant. Any $\bfD \in \mcG$ will satisfy the following properties, each of which we will use frequently:
\begin{enumerate}[label=$\mcC$.\arabic*]
    \item \label{C1:D_norm} $cK/M \leq \left\|\bfD\bfD^T\right\|_2 \leq CK/M$.
    \item \label{C2:D_norm_F} $cK/\sqrt{M} \leq \left\|\bfD\bfD^T\right\|_F \leq CK/\sqrt{M}$.
    \item \label{C3:y_norm} For all $i = 1,\ldots, N$, $\sqrt{s}/2 \leq \|\bfy_i\|_2 \leq 3\sqrt{s}/2$
\end{enumerate}

\section{Algorithm}\label{algorithm}

In this section, we outline the key elements of the \textbf{SPORADIC} (\textbf{SP}arse \textbf{ORA}cle \textbf{DIC}tionary Learning) algorithm. The essential idea is the following: as already noted, in the linear-sparsity regime, correlations are not strong enough to directly infer the sparsity pattern as in \citep{Arora13}. Therefore in this work we propose an intermediate step: before attempting to recover the dictionary elements, for each sample $\bfy_i$ we recover its \textit{spanning subspace} $\mcS_i$:

\begin{definition}[Spanning Subspace]
Given a sample $\bfy_i = \bfD\bfx_i$, the \emph{spanning subspace} of sample $i$ is the subspace $\mcS_i$ defined as
$\mcS_i = \vspan\{\bfd_k: k \in \supp{\bfx_i}\}$.
\end{definition}

First recovering the spanning subspaces obviates any need to perform community detection on an unreliable connection graph, which was the immediate point of failure for the correlation-based method of \citet{Arora13} in the $s\gg\sqrt{M}$ setting.

\subsection{Algorithm Overview}

\textbf{SPORADIC} consists of three steps: first, an initial estimate is learned by means of a correlation-based spectral method, \textbf{Spectral Subspace Dictionary Larning (SSDL)}. Next, the dictionary elements found in this initial estimate are used as ``oracles'' to more accurately estimate the true support of the sparsity pattern $\bfX$ and then to recover a correspondingly better estimate of $\bfD$; this is the \textit{Oracle Refinement} step. Both of these estimates have dimensional error terms which will not tend to zero regardless of the number of sample, for which reason we employ the third and final step, \textit{Oracle Averaging}. Here we take advantage of the fact that the oracle refinement step results in an estimate of $\bfD$ which is close enough that not only the support of $\bfX$ but also the signs of its nonzero entries can be detected, after which dictionary elements $\bfd_k$ can be recovered by averaging the sign-adjusted samples $\bfy_i$ with $k \in \Omega_i$. Of the three steps, obtaining an initial estimate is the most challenging and correspondingly constitutes the bulk of our analysis. The process is summarized below:
\begin{enumerate}
\item \textbf{Spectral Subspace Dictionary Learning (SSDL)}: Learn a rough estimate of the dictionary via a spectral method.
	\begin{enumerate}
	\item \textit{Subspace Recovery}: Recover estimated spanning subspaces $\hmcS_i$ of each sample $\bfy_j$ by computing the the span of the top $s$ eigenvectors of the \textit{correlation-weighted covariance} $\frac{1}{N}\sum_{i=1}^N\inner{\bfy_j}{\bfy_i}^2\bfy_i\bfy_i^T$ and adjusting for the population covariance $\frac{1}{N}\bfY\bfY^T$.
	\item \textit{Subspace Intersection}: Compare subspaces $\hmcS_i$ by taking approximate intersections between collections of subspaces; approximate 1-dimensional intersections are included as estimated dictionary vectors $\hbfd_k$.
	\end{enumerate}
\item \textbf{Oracle Refinement}: Learn the estimated support $\tO_i$ of each sample $\bfx_i$ by testing whether the size of the projection of $\hbfd_k$ onto the estimated spanning subspace $\hmcS_i$ exceeds a constant threshold for each $i \in \{1,\ldots,N\}$ and $k\in \{1,\ldots,K\}$. Use support information to recover a refined estimate $\tbfd_k$ using a spectral method similar to SSDL.
\item \textbf{Oracle Averaging}: From $\tbfd_k$, deduce the signs of nonzero entries $\bfX$ and average the sign-corrected samples $x_{ik}\bfy_i$ for each $i$ such that $k\in\tO_i$.
\end{enumerate}

\subsection{Spectral Subspace Dictionary Learning}
We begin with the Spectral Subspace Dictionary Learning (SSDL) algorithm for recovering an initial estimate of the dictionary. SSDL itself consists of two steps: \textit{subspace recovery}, in which the spanning subspaces of each sample $\bfy_i$ are approximately recovered, and \textit{subspace intersection}, which compares the recovered subspaces to learn shared dictionary elements.

\subsection{Subspace Recovery}
Given a sample $\bfy_j$ the subspace recovery step is a spectral method that constructs a matrix which, with high probability, will have lead $s$ eigenvectors spanning a subspace close to the true spanning subspace $\mcS_j$. To estimate $\mcS_j$, we consider a statistic based on the classical estimator for the covariance of $\bfy$, the \textit{sample covariance matrix} $\HSig$:
\[
\HSig = \frac{1}{N}\bfY\bfY^T = \frac{1}{N}\sum_{i=1}^N\bfy_i\bfy_i^T
\]
As long as $E\bfy = 0$, it is easy to see that $\HSig$ is an unbiased estimator (that is, $E\HSig = E\bfy\bfy^T$); by the law of large numbers, then, $\HSig \to E\bfy\bfy^T$ in $N$ almost surely (later we will use quantitative versions of this result; see, for instance, \citet{Vershynin}, theorems 4.7.1 and 5.6.1).

To find the subspace $\mcS_j$, though, we need an estimator which is biased towards those directions spanned by the support elements of $\bfy_j$. Our goal is to weight the sample covariance matrix in such a way that a sample $\bfy_i$ is given more weight the larger the shared support between $\bfy_i$ and $\bfy_j$, and therefore the resulting covariance will be ``stretched'' in favor of the directions spanned by the support elements of $\bfy_j$.. Accordingly, we now introduce the key statistic of the subspace recovery step, the \textit{correlation-weighted covariance} $\Sig_j$:
\[
\Sig_j := E[\inner{\bfy_j}{\bfy}^2\bfy\bfy^T]
\]
and its sample version $\HSig_j$:
\[
\HSig_j := \frac{1}{N}\sum_{i \neq j}^N \inner{\bfy_j}{\bfy_i}^2 \bfy_i\bfy_i^T
\]
We point out that $\Sig_j$ and $\HSig_j$ are the covariance and the sample covariance estimator, respectively, for the random vector $\inner{\bfy_j}{\bfy}\bfy$. 

A major challenge is that when $s\gg \sqrt{M}$ this bias remains small, with the result that $\HSig_k$ will be close to the unweighted covariance matrix of $\bfy$, with only a small perturbation in the directions $\mcS_j$. However, this covariance can be accurately estimated by the sample covariance $\HSig = \frac{1}{N}\bfY\bfY^T$. With this estimate in hand, we remove the $\bfD\bfD^T$ component by ``covariance projection:'' taking the orthogonal complement of $\HSig_j$ (in the Frobenius sense) with respect to the unweighted sample covariance matrix $\HSig = \frac{1}{N}\sum_{i=1}^N \bfy_i\bfy_i^T = \frac{1}{N}\bfY\bfY^T$, with the goal of leaving only the bias component. Thus for our spectral method, we ultimately look at the span of the eigenvectors corresponding to the $s$ largest eigenvalues of the matrix
\[
\HSig_j^{\text{proj}} := \HSig_j - \frac{\inner{\HSig_j}{\HSig}_F}{\|\HSig\|^2_F}\HSig
\]
As sample covariance matrices, $\HSig_0$ and $\HSig$ can be made arbitrarily close to their expectations with sufficiently large sample size, so this statistic will have spectral properties close to the bias matrix $\sum_{k\in\Omega_j}\bfd_k\bfd_k^T$, along with a spike in the direction of $\bfy_0$, which causes no issue since $\bfy_0 \in \mcS_0$. We detail the precise process in Algorithm \ref{subspace_recovery_algo}.

\begin{algorithm}[h]
\caption{\textbf{SSR}: Single Subspace Recovery} \label{subspace_recovery_algo}
\textbf{Input:} index $j$, sample matrix $\bfY \in \R^{M\times N}$, est. covariance matrix $\HSig = \frac{1}{N}\bfY\bfY^T$ \\
\textbf{Output:} $s$-dimensional subspace $\hmcS_j$\\
\textit{Correlation-Weighted Covariance:} Compute $\HSig_j = \frac{1}{N}\sum_{i = 1}^N \inner{\bfy_j}{\bfy_i}^2\bfy_i\bfy_i^T$ \\
\textit{Covariance Projection:} Compute $\HSig^{\text{proj}}_j = \HSig_{j} - \proj{\HSig}{\HSig_j}$ \\
\textit{Spectral Recovery:} Compute the leading $s$ eigenvectors of $\HSig_j$ and set $\hmcS_j$ equal to their span. \\ 
\Return $\hmcS_j$
\end{algorithm}

Naturally, the subspace $\hmcS_j$ recovered by Algorithm \ref{subspace_recovery_algo} will only approximately match the true subspace $\mcS_j$. For this reason we introduce the following metric on subspaces of the same dimension:
\begin{definition}[Subspace Distance]\label{subspace_dist}
For two $s$-dimensional subspaces $\mcS_1$, $\mcS_2$ of $\R^M$, let $\bfE_1$ be an orthonormal basis of $\mcS_1$ and let $\bfF_2$ be an orthonormal basis of $\mcS_2^\perp$, the orthogonal complement of $\mcS_2$. We define the \emph{subspace distance} $\mcD$ between $\mcS_1$ and $\mcS_2$ as:
\[
\mcD(\mcS_1, \mcS_2) = \sup_{\bfz \in \mcS, \|\bfz\|_2 = 1}\|\bfz - \proj{\mcS_2}{\bfz}\|_2 =  \sup_{\bfz \in \mcS, \|\bfz\|_2 = 1}\|\proj{\mcS_2^\perp}{\bfz}\|_2 
 = \left\|\bfF_2\bfF_2^T\bfE_1\right\|_2 = \|\bfF_2^T\bfE_1\|_2
\]
\end{definition}
In Section \ref{scn:sub_recov_guarantee}, we demonstrate that the recovered $\hmcS_j$ is close to the true $\mcS_j$ for each $j$ simultaneously with high probability. We also show in Theorem \ref{int_lemma} that with high probability, up to logarithmic factors only the first $\bigO{K}$ subspaces are required to find every dictionary vector via subspace intersection. 

\subsection{Subspace Intersection}
Since the eigenvectors returned by this spectral method are orthonormal, they do not correspond to dictionary elements directly, but instead form a basis for a subspace $\hmcS_i$ close to the true subspace $\mcS_i$. Thus an additional \textit{subspace intersection step} is needed to recover actual dictionary elements from the estimated subspaces.

To motivate our subspace intersection algorithm, we note that if the subspaces $\mcS_i$ were known exactly, there is a particularly simple algorithm to find a dictionary element when subspaces $\mcS_i$ are known exactly. Since $\mcS_i = \vspan\{\bfd_k\}_{k\in\Omega_i}$, $\mcS_i \cap \mcS_j = \vspan\{\bfd_k\}_{k \in \Omega_i \cap \Omega_j}$ (almost surely). It follows that if $\dim(\mcS_i \cap \mcS_j) = 1$ exactly, then $\mcS_i \cap \mcS_j = \vspan\{\bfd_k\}$ where $k$ is the unique element in $\Omega_i \cap \Omega_j$.

Letting $\bfF_i$ and $\bfF_j$ be orthonormal basis matrices for $\mcS_i$ and $\mcS_j$, respectively, we can write an element in $\mcS_i$ as $\bfF_i\bfv$ for $\bfv \in \R^s$. Since the matrix for projection onto subspace $\mcS_j$ is $\bfF_j\bfF_j^T$, it follows that $\dim(\mcS_i \cap \mcS_j) = \dim(\ker(\bfF_i - \bfF_j\bfF_j^T\bfF_i))$; denote this matrix $\bfP_{ij}$. Then any $\bfv$ in the kernel of $\bfP_{ij}$ corresponds to a vector $\bfF_i\bfv$ in $\mcS_i \cap \mcS_j$. Then if $\dim({\ker({\bfP_{ij}})}) = 1$, we can easily recover a basis for the intersection $\mcS_i \cap \mcS_j$.

As long as $s^2/K$ is small, $|\mcS_i \cap \mcS_j|$ will typically have either 0 or 1 element, so an intersection between two subspaces will rarely have dimension above one. When $s^2/K \gg 1$---typically the case in our setting---we will instead need to perform subspace intersection on more than two subspaces. It is easy to see that:
\[
E|\Omega_1 \cap \Omega_2 \cap \ldots \Omega_\ell| = \frac{s^{\ell+1}}{K^\ell}
\]
Therefore, for $\ell \geq \frac{\log{s}}{\log{K/s}}$, we have $E|\Omega_0 \cap \Omega_2 \cap \ldots \Omega_\ell] \leq 1$ (this bound is made precise in Theorem \ref{int_lemma} with a slightly larger $\ell$).

In the case when we only know approximate subspaces $\hmcS_i$, $\bfP_{ij}$ will almost surely have trivial kernel, so we relax the condition $\dim(\ker(\bfP_{ij})) = 1$ to the condition that, given some small threshold $\uptau$, $\bfP_{ij}$ has exactly one singular value $\sigma \leq \uptau$. This works under the assumption that dictionary elements are nearly orthogonal, which holds when columns of $\bfD$ are i.i.d. random vectors for a broad class of random vectors \citep[e.g.,]{Vershynin}. In practice, $\uptau$ should not need to be very small; $\uptau = 1/2$ was adequate in numerical experiments. We thus define the \textit{approximate subspace intersection} of two subspaces $\mcS_i$, $\mcS_j$ as follows:

\begin{definition}[Approximate Subspace Intersection]
Let $\mcS_i$ and $\mcS_j$ be subspaces of $\R^M$ with respective orthonormal basis matrices $\bfF_i, \bfF_j$. Denote by $\bfP_{ij} = (\bfI - \bfF_j\bfF_j^T)\bfF_i$ the projection matrix of $\mcS_i$ onto $\mcS_j^\perp$. The \emph{approximate subspace intersection} of $\mcS_i$ onto $\mcS_j$ with threshold $\uptau$ is the subspace $\mcA_\uptau(\mcS_i, \mcS_j)$ of $\R^M$ defined as the span of all right singular vectors of $\bfP_{ij}$ corresponding to sufficiently small singular values:
\[
\mcA_\uptau(\mcS_i, \mcS_j) = \vspan{\{\bfv : \bfv \text{ is a right singular vector  of } \bfP_{ij} \text{ corresponding to a singular value } \sigma_\bfv \leq \uptau\}}
\]
with the convention that $\vspan(\emptyset) = \{\bfzo\}$.
\end{definition}

\noindent We note that this definition ensures $\mcA_\uptau(\mcS_i, \mcS_j)$ is a subspace of $\mcS_i$.

In Algorithm \ref{subspace_int_approx}, we detail the approximate subspace intersection algorithm for a fixed list of $\ell$ subspaces $\mcS_{i_1}, \ldots, \mcS_{i_\ell}$. In Theorem \ref{int_lemma}, we show that with high probability that to recover all $K$ dictionary elements it suffices to consider only the non-overlapping intersections $\bigcap_{p = 1}^\ell \mcS_{\ell(j-1)+p}$ for $j = 1, \ldots, J/\ell$ with $J \approx K\log^2K$. To recover an entire dictionary, then, we first employ subspace recovery to learn the subspaces $\{\hmcS_1,\ldots, \hmcS_J\}$ for the first $K\log^2 K$ samples, then take the intersection of each consecutive set of $\ell$ subspaces. (Duplicates, those estimated dictionary elements which are close to one another based on absolute inner product, can be safely rejected.)

\begin{algorithm}[h]
\caption{\textbf{SSI}, Approximate Subspace $\ell$-fold Intersection} \label{subspace_int_approx}
\textbf{Input:} List of subspaces $\mcS_{1},\ldots, \mcS_{\ell}$, threshold $\uptau < 1$\\
\textbf{Output:} Estimated dictionary element $\hbfd$, or \textbf{False} if no element is found.\\
$\mcS \gets \mcS_1$\\
\For{$i \in \{2, \ldots, \ell\}$}
{
    $\mcA \gets \mcA_\uptau(\mcS, \hmcS_i)$\\
    \uIf{$\dim(\mcA)= 0$}
    {
        \Return \textbf{False}
    }
    \uElseIf{$\dim(\mcA) = 1$}
    {
        Set $\hbfd$ a basis of $\mcA$\\
        \Return $\hbfd$
    }
    \uElseIf{$\dim(\mcA) \geq 2$}
    {
        $\mcS = \mcA$
    }
\Return \textbf{False}
}
\end{algorithm}

\subsection{Oracle Refinement}
The subspace recovery algorithm gives a nontrivial column-wise estimate $\hbfD$ of the dictionary $\bfD$. On its own, however, this estimate is unsatisfactory for several reasons. First, for $s$ nearly linear in $M$, the resulting error bound decays logarithmically in $M$, barely better than a constant error. Moreover, a logarithmic column-wise estimate of does not provide a small error bound in terms of any matrix norm, meaning the behavior of the recovered dictionary could in fact be quite different from that of the true dictionary. Lastly, the resulting error bound contains a dimensional error that does not tend to zero with $N$, meaning no amount of samples will ever provide a perfect estimate of $\bfD$. 

The oracle refinement step takes the initial estimate from SSDL and boosts its accuracy enough to provide meaningful improvements in the first two cases; resolving the last requires the additional oracle averaging step. To this end, oracle refinement uses the recovered columns of $\hbfD$ along with the estimated spanning subspaces $\hmcS_i$ to deduce the support of each sample $\bfy_i$. For each $k\in\{1,\ldots,K\}$ and $i\in\{1,\ldots,N\}$, we consider the size of the projection of each estimated column $\hbfd_k$ onto each subspace $\hmcS_i$. We show that as long as the initial recovery error $\varepsilon$ is smaller than a constant threshold, this projection will be greater than $1/2$ if and only if $k \in \Omega_i$. In other words, a small enough constant column-wise error is sufficient to recover the entire support of $\bfX$.

We can then apply a similar spectral method, including the covariance projection step. Specifically, letting $\tA_k = \{i\in\{1,\ldots,N\}: \|\bfP_{\hmcS_i} \hbfd_k\|_2 > 1/2\}$ we compute the matrices
\[
\tbfV_k = \frac{1}{|\tA_k|}\sum_{i\in\tA_k}\bfy_i\bfy_i^T
\]
and their projected versions
\[
\pbfV_k = \tbfV_k - \proj{\HSig}{\tbfV_k}.
\]
We then create a new estimated dictionary, $\tbfD$, with columns $\tbfd_k$ given by the lead eigenvector of $\pbfV_k$. By the same intuitive reasoning that underpinned the use of correlation-weighted covariance, since every summand contains $\bfd_k$ in its support, the resulting covariance will be stretched in the $\bfd_k$ direction. The resulting error in column-wise terms reduces from logarithmic decay to decaying faster than $1/\sqrt{s}$. As we will see in the following section, this is a critical error cutoff for dictionary recovery, as below this threshold it is possible to recover not only the support of $\bfX$ but even the \textit{signs} of its nonzero entries.

\begin{algorithm}[h]
\caption{\textbf{Oracle Refinement}} \label{algo:oracle_refinement}
\textbf{Input:} Estimated dictionary $\hbfD \in \R^{M\times K}$, sample matrix $\bfY \in \R^{M\times N}$, estimated spanning subspaces $\{\hmcS_i\}_{i=1}^N$, estimated covariance $\HSig \in \R^{M\times M}$, threshold $\uptau \in (0,1)$\\
\textbf{Output:} Estimated support sets $\{\tO_i\}_{i=1}^N$, estimated dictionary $\tbfD$\\
$\tbfD = \bfzo^{M\times K}$\\
$\tO_i = \emptyset$ for all $i=1,\ldots, N$\\
\For{$k \in \{1, \ldots, K\}$}
{
    \For{$i \in \{1, \ldots, N\}$}
    {
        $\bfP_{\hmcS_i} \gets $ orthogonal projection matrix for $\hmcS_i$\\
        \uIf{$\big\|\bfP_{\hmcS_i}\hbfd_k\big\|_2 > \uptau$}
        {
            $\tO_i \gets \tO_i \cup \{k\}$
        }
    }
    $\tA_k \gets \{i = 1,\ldots,N: k\in\tO_i\}$\\
    $\tbfV_k \gets \frac{1}{|\tA_k|}\sum_{i\in \tA_k}\bfy_i\bfy_i^T$\\
    $\pbfV_k \gets \tbfV_k - \proj{\HSig}{\tbfV_k}$\\
    $\tbfd_k \gets$ lead normalized eigenvector of $\pbfV_k$
}
\Return $\{\tO_i\}_{i=1}^N$, $\tbfD$\\
\end{algorithm}

\subsection{Oracle Averaging}
We conclude with the simplest of our recovery algorithms, the \textit{oracle averaging} procedure. With the column-wise estimation error now decaying faster than $1/\sqrt{s}$, it becomes possible to recover the signs of nonzero entries of $\bfx_i$ by looking at the inner product between $\tbfd_k$ and $\bfy_i$ for each $i$ with $k\in\Omega_i$. 

Once the signs are known, simple averaging becomes possible by correcting the signs of each sample containing $k$ in its support:
\[
\bbfd_k = \frac{1}{|\tA_k|}\sum_{i \in \tA_k}\sign{\inner{\tbfd_k}{\bfy_i}}\bfy_i
\]
and error bounds can be derived by standard arguments.

The advantage of oracle averaging is that, where the SSDL and oracle refinement steps have a dimensional component to their error bounds that does not decay with number of samples, the error bounds of oracle averaging decays to zero with additional samples. Thus in data-rich applications, the guarantees for the oracle averaging procedure are substantially tighter than those for SSDL or oracle refinement on their own.

\begin{algorithm}[h]
\caption{\textbf{Oracle Averaging},} \label{algo:oracle_averaging}
\textbf{Input:} Estimated dictionary $\tbfD \in \R^{M\times K}$, sample matrix $\bfY \in \R^{M\times N}$, estimated supports $\{\tO_i\}_{i=1}^N$ \\
\textbf{Output:} Estimated dictionary $\bbfD$\\
$\bbfD = \bfzo^{M\times K}$\\
\For{$k \in \{1, \ldots, K\}$}
{
    $\tA_k \gets \{i = 1,\ldots,N: k \in \tO_i\}$\\
    $\bbfd = \frac{1}{|\tA_k|}\sum_{i \in \tA_k}\sign{\inner{\tbfd_k}{\bfy_i}}\bfy_i$\\
    $\bbfd_k \gets \bbfd$
}
\Return $\bbfD$
\end{algorithm}

\subsection{Time Complexity}

To compute the correlation-weighted covariance $\HSig_j$ for a single $j$, the single subspace recovery algorithm adds $N$ matrices of the form $\inner{\bfy_j}{\bfy_i}^2\bfy_i\bfy_i^T$, each of which can be computed in $\bigO{M^2}$ time, meaning $\HSig_j$ can be computed in time $\bigO{NM^2}$. Finding the top $s$ eigenvalues and eigenvectors then takes an additional $\bigO{sM^2}$ operations, meaning the entire subspace recovery step for a single sample $\bfy_i$ can be completed in $\bigO{NM^2}$ time. Recovering all $N$ subspaces means the entire process will take time $\bigO{N^2M^2}$ time.

The runtime of the each subspace intersection step is dominated by matrix multiplication and finding eigenvectors, which both have order $\bigO{M^3}$. Under the assumption that the support of each sample is uniformly distributed among $s$-element subsets of $\{1,\ldots,K\}$, we show in Theorem \ref{int_lemma} that with high probability, one only needs to check fewer than $\bigO{K\log^2 K}$ intersections in order to recover each dictionary element. Accordingly, with high probability the subspace intersection step to take $\bigO{\ell KM^3}$ time, so we conclude that the entire SSDL process takes $\bigO{N^2M^2 + KM^3}$ with high probability, up to log factors.

The oracle refinement step requires $NK$ projection thresholding steps to estimate the supports $\tO_i$, each of which will take $M^2s$ time to compute for a total of $\bigO{NKM^2s}$ time. The each of the $K$ subsequent spectral method will take time $N_kM^2 \approx NM^2s/K$ to compute the covariance matrices, and time $M^2$ to compute the top eigenvectors, totaling $\bigO{NM^2s}$ time. Lastly, the oracle averaging step will require $K$ averages of $N_k \approx Ns/K$ $M\times 1$ vectors for total time of $\bigO{NMs}$. Accordingly, the total runtime for \textbf{SPORADIC} is $\bigO{N^2M^2 + KM^3 + NM^2s}$. With a sample complextiy of $N \sim M^4$ in the linear sparsity regime, the dominant term in this runtime will be $\bigO{N^2M^2}$ from the SSDL step.

\section{Main Results}\label{results}

In this section, we state our main result along with several component results.

\begin{thm}\label{master_thm}
Let $s < M < K$ with $K/M^{3/2} \leq \sqrt{K}/\sqrt{M}$, and let $\bfY = \bfD \bfX$ with $\bfD \in \mcG$ and $\bfX \sim \mcX(W)$. For any $\alpha \geq 1$, assume that $M$ is sufficiently large and that
\[
\varepsilon_{\emph{\text{SSDL}}} = \frac{\sqrt{M}\log M}{\sqrt{K}} + \frac{\alpha 
 Ks\log^4 M}{M^2} + \frac{Ks^2\sqrt{\alpha\log^5 M}}{M^{3/2}\sqrt{N}} + \frac{s^5\sqrt{\alpha\log^3 M}}{M^3\sqrt{N}} \leq c
\]
and 
\[
\varepsilon_{\emph{\text{oracle}}} = \frac{K}{M^{3/2}} + \frac{\sqrt{K\alpha M \log M}}{\sqrt{Ns}} + \frac{\alpha M\log M}{K\sqrt{N}} \leq \frac{c}{\sqrt{s}}
\]
where $c$ is a small absolute constant. Then there exists an absolute constant $C > 0$ such that:
\begin{enumerate}
\item \textbf{\emph{SSDL}} with $\ell =\left\lceil\log(2K)/\log(K/s)\right\rceil$, $J \geq 4K(\alpha+1)\ell\log K$, and $\uptau = 1/2$ recovers a dictionary $\hbfD$ that is column-wise $(C\varepsilon_{\emph{\text{SSDL}}})$-close to $\bfD$ with probability at least
\[
1-2N^2\exp(-s^2/K) -K^{-\alpha} - 3NKM^{-\alpha}.
\]
\item \textbf{Oracle refinement} recovers a dictionary $\tbfD$ that is column-wise $(C\varepsilon_{\emph{\text{oracle}}})$-close to $\bfD$ with probability at least
\[
1-2N^2\exp(-s^2/K) - K\exp(-Ns/(10K))- K^{-\alpha} - 3NKM^{-\alpha}.
\]
\item \textbf{Oracle averaging} recovers a dictionary $\bbfD$ that is column-wise 
\[
\left(\frac{3}{2}+3\sqrt{\alpha\log M}\right)\frac{\sqrt{K}}{\sqrt{N}}\text{-close}
\]
to $\bfD$ with probability at least  
\[
1-2N^2\exp(-s^2/K) - K\exp(-Ns/(10K))- K^{-\alpha} - 3NKM^{-\alpha}.
\]
\end{enumerate}
\end{thm}

\noindent The key takeaway from this result is that in high enough dimensions, dictionary recovery is possible for linear sparsity in $M$ up to log factors with arbitrarily small error given enough samples:

\begin{corollary}
In the setting of Theorem \ref{master_thm}, suppose that there exist constants $0 < \gamma < \eta$ such that
\begin{itemize}
    \item $s = M\log^{-(6  + 2\eta)} (M)$
    \item $K = M\log^{(2 + \gamma)} (M)$
    \item $N \geq \frac{3\alpha}{c}\max\left\{\frac{s^{10}\log^3 M}{M^6}, \frac{K^2s^4\log^5 M}{M^3}\right\}$.
\end{itemize}
Then for $M$ sufficiently large, bounds 1-3 in Theorem \ref{master_thm} hold with probability at least $1 - \bigO{NKM^{-\alpha}}$. 
\end{corollary}

This scaling guarantees that $\varepsilon$ as in Theorem \ref{master_thm} will be small for large enough $M$ as well as ensuring $K/M^{3/2} \ll 1/\sqrt{s}$. The $Ks\log^4 M/M^2$ term suggests a tradeoff between sparsity and overcompleteness: the more overcomplete the dictionary, the sparser the representation must be to enable recovery. This is not a huge limitation, as it stands to reason that a larger dictionary should allow for a sparser representation.

This corollary indicates that the sample complexity for recovery is 
\[
N \sim \max\left\{\frac{s^{10}\log^3 M}{M^6}, \frac{K^2s^4\log^5 M}{M^3}\right\},
\]
which is approximately $N \sim M^4$ in the most challenging linear-sparsity setting. This result also suggests that the sample complexity required by SSDL is lower in easier sparsity regimes. The high sample complexity does guarantee that in the regime for which oracle averaging works, the resulting error bound is tighter than that for either SSDL or oracle refinement, even without any additional samples.

\subsection{Guarantees for SSDL}

We now outline more detailed results for the individual components of the \textbf{SPORADIC} algorithm, which collectively constitute Theorem \ref{master_thm}. As previously noted, the majority of the analysis will be dedicated to proving the first of these items, which claims that SSDL gives a nontrivial estimate of $\bfD$. The following theorem states that under the scaling noted above, SSDL returns a nontrivial estimate of the dictionary, on which the subsequent recovery will build.

\begin{thm}[SSDL]\label{thm:ssdl}
In the setting of Theorem \ref{master_thm}, \textbf{\emph{SSDL}} with parameters $\ell =\left\lceil\log(2K)/\log(K/s)\right\rceil$, $J \geq 4K(\alpha+1)\ell\log K$, and $\uptau = 1/2$ recovers a dictionary $\hbfD$ that is column-wise $(C\varepsilon)$-close to $\bfD$ with probability at least
\[
1-2N^2\exp(-s^2/K) -K^{-\alpha} - (2K+4)NM^{-\alpha}.
\]
\end{thm}

\subsection{Subspace Recovery}
Theorem \ref{thm:ssdl} follows from two results, providing independent guarantees for the subspace recovery and subspace intersection steps. We begin with the following guarantee for the accuracy of the subspace recovery step:

\begin{thm}[Subspace Recovery]\label{thm:subspace_recovery}
In the setting of Theorem \ref{master_thm}, let $\mcS_j$ be the spanning subspace of sample $\bfy_j$ and $\hmcS_j$ be the subspace recovered by Algorithm \ref{subspace_recovery_algo}. Then there exists absolute constant $c'$ such that if
\[
\tilde{\varepsilon} := \frac{\varepsilon}{\log M} = \frac{\sqrt{M}}{\sqrt{K}} + \frac{\alpha Ks\log^3 M}{M^2} + \frac{Ks^2\sqrt{\alpha \log^3 M}}{M^{3/2}\sqrt{N}} + \frac{s^5\sqrt{\alpha\log M}}{M^3\sqrt{N}} \leq c',
\]
we have
\[
\mcD\left(\mcS_j, \hmcS_j \right) \leq  C\tilde{\varepsilon}.
\]
for all $j \in \{1,\ldots,N\}$ with probability at least $1-2N^2\exp(-s^2/K) -K^{-\alpha} - (2K+4)NM^{-\alpha}$.
\end{thm}

Details of the proof follow in Section \ref{scn:proof}, but the key ingredients are these: by a union bound, it suffices to prove that the desired bound holds for a single $j$. Using concentration of measure results, we show that $\HSig_j$ converges to its expectation, which will be close to a rank-$s$ matrix with eigenvectors spanning $\mcS_j$. We then apply Weyl's Theorem and the Davis-Kahan Theorem on continuity of eigenvalues and invariant subspaces, respectively, of symmetric matrices under perturbation \citep{Weyl, DavisKahan}, allowing us to bound the distance between the recovered subspace $\hmcS_j$ and true subspace $\mcS_j$.

\subsection{Subspace Intersection}
We now present guarantees stating that that with high probability, the subspace intersection step rejects groups of subspaces which do not contain a unique dictionary element in their intersection. If they do intersect, then subspace intersection returns a vector close to the true vector. Specifically:
\begin{thm}[Subspace Intersection]\label{intersection_thm}
Let $\ell =\left\lceil\log{2K}/\log{K/s}\right\rceil$ and suppose that for all $i \in \{1,\ldots,N\}$, 
\[
\mcD(\mcS_i, \hmcS_i) \leq \varepsilon \leq 1/4.
\]
Then for all $\ell$-element subsets $\mcI$ of $\{1,\ldots.N\}$, if $\bigcap_{i \in \mcI} \mcS_i = \vspan(\bfd_k)$, then Algorithm \ref{subspace_int_approx} with $\uptau = 1/2$ will return a vector $\hbfd$ satisfying:
\[
\min_{t \in \{-1,1\}}\{\|\hat{\bfd} - t\bfd_k\|_2\} \leq C\ell\varepsilon \leq C\varepsilon\log M
\]
Moreover, if $\mathbf{dim}\left(\bigcap_{i \in \mcI} \mcS_i\right) \neq 1$, the algorithm returns \textbf{False}.
\end{thm}

\noindent This result follows from Theorem \ref{thm:subspace_recovery} and $2s$-RIP of $\bfD$ (\ref{G2:s_rip}).
To complete the accuracy guarantees of Theorem \ref{thm:ssdl}, we show that it is possible to isolate each dictionary element by looking at only a polynomial number of $\ell$-element intersections. This ensures we can recover every column in the dictionary in polynomial time with high probability. Specifically, for $\ell = \left\lceil\frac{\log{2K}}{\log{K/s}}\right\rceil$ and $J \geq  CK\ell\log K$, we claim that with high probability, choosing $\mcJ$ in Theorem \ref{intersection_thm} to be the first $J/\ell$ disjoint $\ell$-element subsets is sufficient to recover every dictionary element at least once.

To complete the accuracy guarantees of Theorem \ref{thm:ssdl}, we conclude by showing that it is possible to isolate each dictionary element by looking at only a polynomial number of $\ell$-element intersections. This ensures we can recover every column in the dictionary in polynomial time with high probability. Specifically, for $\ell = \left\lceil\frac{\log{2K}}{\log{K/s}}\right\rceil$ and $J \geq  CK\ell\log K$, we claim that with high probability, choosing $\mcJ$ in Theorem \ref{intersection_thm} to be the first $J/\ell$ disjoint $\ell$-element subsets is sufficient to recover every dictionary element at least once.

\begin{restatable}[Polynomially-Many Intersections Suffice]{thm}{intLemma}\label{int_lemma}
In the setting of Theorem \ref{master_thm}, let $\ell=\left\lceil\log{2K}/\log{K/s}\right\rceil$, that is, the smallest integer such that $(s/K)^\ell < 1/2K$. For positive integer $J$ and $j\in \{1,\ldots, J/\ell\}$, define the non-intersecting $\ell$-fold intersections $\Omega^{\ell}_j$ as:
\[
\Omega^{\ell}_j = \bigcap_{p=1}^\ell \Omega_{(j-1)\ell + p}
\]
Then for any constant $\alpha$, as long as $J \geq 4(\alpha+1) K\ell \log K \approx \alpha K\log M\log K$,  for every $k \in \{1,\ldots,K\}$ there exists a $j \in \{1, \ldots, J/\ell\}$ such that $\Omega^{\ell}_j = \{k\}$ with probability at least $1-K^{-\alpha}$.
\end{restatable}

Theorem \ref{int_lemma} can be proven by fixing $k$ and noting that since the sets $\Omega^\ell_j$ do not overlap, they will be independent, then calculating the probability that none of them contain $k$ as a unique element of intersection; a union bound completes the proof. Details are in Appendix \ref{appdx:int_lemma}.

Theorem \ref{int_lemma} guarantees that subspace intersection will recover every dictionary element at least once while requiring only order $K\log^2 K$ intersections. Theorem \ref{intersection_thm} ensures subspace intersection (Algorithm \ref{subspace_int_approx}) will correctly detect overlapping support with high probability, and will recover the associated dictionary element up to error $\epsi$ by Theorem \ref{thm:subspace_recovery}. Taken together, these results complete the proof of Theorem \ref{thm:ssdl}.

\subsection{Oracle Refinement and Oracle Averaging}

Before stating the main theorems in this section, we note the following fact that the number of samples with $k$ in their support will reliably be of order $Ns/K$ up to absolute constants. Namely,
we denote $N_k = \{i \in\{1,\ldots,N\} : k \in \Omega_i\}$. It is easy to see that $EN_k = Ns/K$, and thus is straightforward to prove the following lemma using a Chernoff bound:
\begin{lma}\label{lemma:nk}
For all $k$, $N_k \geq \frac{Ns}{2K}$ with probability at least $1-K\exp(-Ns/(10K))$.
\end{lma}
The proof is a trivial application of a Chernoff bound and is omitted.

We now state the main theorem for oracle refinement. This theorem states that when the initial estimate from SSDL is close enough, the oracle refinement step drastically improves column-wise error bounds from $Ks\log^4 M/M^2$---nearly constant---to $K/M^{3/2}$, less than $1/\sqrt{s}$. In these results, we assume that the initial estimate gives nontrivial estimation of $+\bfd_k$; in practice it is possible that we have an estimate of $-\bfd_k$ instead, but this has no impact on results.

\begin{restatable}{thm}{oracleRef}\label{thm:oracle_ref}
Let $\bfD\in\mcG$ and let $\varepsilon < 1/8$ be such that $\mcD(\mcS_i, \hmcS_i) \leq \varepsilon$ for all $i$ and that 
$\|\hbfd_k - \bfd_k\|_2 \leq \varepsilon$ for all $k$. Then for all $i$, 
\[
\tO_i = \{i \in \{1,\ldots,N\}: \|\bfP_{\hmcS_i}\hbfd_k\|_2 > 1/2\} = \Omega_i
\]
and for all $\alpha > 0$,  $\tbfV_k = \frac{1}{|\tO_i|}\sum_{k\in\tO_i} \bfy_i\bfy_i^T$ satisfies
\[
\|\tbfV_k - \proj{\HSig}{\tbfV_k} - (1-s/K)\bfd_k\bfd_k^T\|_2 \leq \frac{CK}{M^{3/2}} + \frac{C\sqrt{\alpha KM \log M}}{\sqrt{Ns}} + \frac{C\alpha M\log M}{K\sqrt{N}}
\]
for all $k$ with probability at least
\[
1-K\exp(-Ns/(10K))-2KM^{-\alpha}.
\]
Accordingly, with the same probability, the columns of the dictionary $\tbfD$ recovered by Algorithm \ref{algo:oracle_refinement} satisfy
\[
\min_{t \in \{-1,1\}}\{\|\tbfd_k - t\bfd_k\|_2\} \leq \frac{CK}{M^{3/2}} + \frac{C\sqrt{\alpha KM \log M}}{\sqrt{Ns}} + \frac{C\alpha M\log M}{K\sqrt{N}}
\]
for all $k$.
\end{restatable}

The resulting $\bigO{1/\sqrt{s}}$ cutoff is critical, as it allows the signs of nonzero entries of $\bfX$ to be recovered in addition to its support. Moreover, by application of the triangle inequality, for any submatrix $\bfD_s$ consisting of $s$ columns of $\bfD$ we have $\|\tbfD_s - \bfD_s\|_2 \leq \|\tbfD_s - \bfD_s\|_F \leq s \max_{k}\|\tbfd_k - \bfd_k\| = \bigO{1}$. Thus if the implied constant is small enough, oracle refinement upgrades an estimator with roughly constant column-wise error to one with a nontrivial estimate on the action of $\bfD$ on all sparse vectors $\bfx$. This allows us to prove the following Theorem on oracle averaging, which states that as long as the estimate $\tbfD_k$ is sufficiently close, oracle averaging achieves ``oracle'' performance in the sense of doing as well as averaging with knowing the support and signs of $\bfX $ exactly. This is the content of the following theorem:

\begin{thm}[Oracle Averaging]\label{thm:oracle_avg}
Let $\bfD \in \R^{M\times K}$ be a dictionary satisfying $\mcG$, and suppose that $\|\tbfd_k - \bfd_k\|_2 \leq 3/(8\sqrt{s})$. Then for any $i$ such that $k \in \Omega_i$, $\sign{\inner{\tbfd_k}{\bfy_i}} = \sign{x_{ik}}$. Moreover,  letting
\[
\bbfd_k = \frac{1}{N_k}\sum_{k\in\Omega_i}x_{ik}\bfy_i,
\]
we have for any constant $\alpha$,
\[
\min_{t \in \{-1,1\}}\{\|\bbfd_k - t\bfd_k\|_2\} \leq \frac{C\sqrt{\alpha K \log M}}{\sqrt{N}}
\]
for all $k$ with probability at least $1-KM^{-\alpha}$.
\end{thm}

As an alternative to the above theorem, it is worth noting that under the assumption that nonzero entries of $\bfX$ are always $\pm 1$, it is actually possible to \textit{exactly} recover the dictionary by inverting a square submatrix of $\bfX$, an approach first suggested in \cite{Agarwal2013ExactRO}. However, unlike our other results, such a method requires that nonzero entries of $\bfX$ be exactly $\pm 1$, and so does not generalize to $\bfX$ with general symmetric bounded nonzero entries. As our $\pm 1$ assumption is a stand-in for an assumption of symmetric random coefficients bounded by $[c,C]$, we proposed the above averaging approach which, while not exact, enjoys error bounds tending to zero in $N$. 

\section{Proofs of Theoretical Guarantees}\label{scn:proof}

In this section, we outline the proofs of our theoretical guarantees from the previous section. In all results in this section, we assume that we are in the setting of Theorem \ref{master_thm} hold, in particular assuming that $\bfD \in \mcG$.

\subsection{Proofs for Subspace Recovery}\label{scn:sub_recov_guarantee}
The most involved of our theoretical guarantees is Theorem \ref{thm:subspace_recovery}, which states that the subspaces recovered by Algorithm \ref{subspace_recovery_algo} are close to the true spanning subspaces. It will suffice to prove the result for a fixed index $j$, as the result will then follow by a union bound over $\{1,\ldots,N\}$. In the proof, we use $j=0$ to indicate this fixed index; although this is a minor abuse of notation as $j$ ranges from $1$ to $N$, this notation emphasizes the distinction between the sample $\bfy_0$ and the others, while the change from $N$ to $N+1$ samples is negligible.

We begin with a broad overview of our approach and the steps involved:
\begin{enumerate}
    \item \textbf{Bound on expectation error.} We compute the expectation $E\HSig_0$ of the correlation-weighted covariance. We then separate the computed expectations into ``signal'' and ``noise'' terms and compute bounds on the noise terms. These bounds depend only on the dimensional parameters $M$, $s$, and $K$.
    \item \textbf{Bound on estimation error.} We bound the probability that the sample $\bfY$ produces a correlation-weighted covariance that is far from its expectation. These bounds are controlled by the number of samples $N$.
    \item \textbf{Subspace Comparison.} We convert the stated bounds on the correlation-weighted covariance to bounds on their $s$-leading subspaces.
\end{enumerate}

\subsubsection{Bounding Expectation Error}

We begin with the following result on the expectations of the correlation-weighted covariance:

\begin{restatable}[Expectation Computation]{lma}{decompLemma}\label{lemma:decompLemma}
Let $\bfv_0 = \bfD\bfD^T\bfy_0$ and $\tbfy_0^k = \bfy_0 - \bbi_{k \in \Omega_0}\bfd_k$. We have
\footnotesize
\begin{multline}\label{decomp_eqn}
E[\inner{\bfy_0}{\bfy}^2\bfy\bfy^T] = \frac{s}{K}\Bigg[\frac{s-1}{K-1}\bfv_0\bfv_0^T + \left(1 - \frac{2(s-1)}{K-1}\right)\sum_{k\in\Omega_0}\bfd_k\bfd_k^T \Bigg] + 
\\ 
+\frac{s}{K}\Bigg[\left(1 - \frac{2(s-1)}{K-1}\right)\left(2\sum_{k\in\Omega_0}^K \inner{\tbfy_0^k}{\bfd_k} \bfd_k\bfd_k^T
+ \sum_{k=1}\inner{\tbfy_0^k}{\bfd_k}^2\bfd_k\bfd_k^T\right) 
+ \left(\frac{s-1}{K-1}\sum_{k =1}^K \inner{\bfy_0}{\bfd_k}^2\right)\bfD\bfD^T \Bigg].
\end{multline}
\normalsize
\end{restatable}
Here the first row consists of signal terms, which form a matrix with $s$-leading eigenvalues approximately spanning the subspace $\mcS_0$. By contrast, the second row consists of nuisance terms which will ultimately not affect this subspace information: they either have small magnitude (the third and fourth terms) or they will be removed by covariance projection (the fifth term). The proof of this lemma is a computation and is deferred to Appendix \ref{appdx:decomp}.

We now bound the difference between our desired biased covariance matrix and the actual expectation of $\HSig_0$, resulting in error bounds intrinsic to the dimension $M$. Specifically, we prove the following lemma:

\begin{lma}[Expectation Error Bound]\label{lemma:exp_detail}
For $\bfv_0 = \bfD\bfD^T\bfy_0$,
\small
\begin{equation}\label{exp_detail_eqn}
\left\|\frac{K}{s}E[\HSig_0] - \left(\frac{s-1}{K-1}\bfv_0\bfv_0^T + \sum_{k \in \Omega_0}\bfd_k\bfd_k^T + \left(\frac{s-1}{K-1}\sum_{k =1}^K \inner{\bfy_0}{\bfd_k}^2\right)\bfD\bfD^T\right)\right\|_2 \leq \frac{CKs\log^3 M}{M^2}
\end{equation}\normalsize
and
\small\begin{equation}\label{exp_detail_cov}
\left\|\frac{K}{s}\left(E\HSig_0 - \proj{\bfD\bfD^T}{E\HSig_0}\right) - \left(\frac{s-1}{K-1}\bfv_0\bfv_0^T + \sum_{k \in \Omega_0}\bfd_k\bfd_k^T\right)\right\|_2 \leq \frac{CKs\log^3 M}{M^2}
\end{equation}
\end{lma}\normalsize
In other words, up to scaling, the expectation of $\HSig_0$ after covariance projection is approximately equal to the low-rank matrix $\frac{s-1}{K-1}\bfv_0\bfv_0^T + \sum_{k\in\OZ}\bfd_k\bfd_k^T$. We begin by bounding the nuisance terms from \ref{decomp_eqn}:
\begin{equation}\label{eqn:nuisance}
2\sum_{k\in\OZ} \inner{\tbfy_0^k}{\bfd_k}\bfd_k\bfd_k^T + \sum_{k=1}^K \inner{\tbfy_0^k}{\bfd_k}^2 \bfd_k\bfd_k^T.
\end{equation}

\noindent We will need the following lemma:
\begin{restatable}{lma}{nuisance}\label{lemma:nuisance}
Let $\alpha$ be any positive constant. There exists an absolute constant $C$ such that, with probability at least $1-2KM^{-\alpha}$,
\[
\left|\inner{\tbfy_0^k}{\bfd_k}\right| \leq \frac{C\sqrt{\alpha s\log^3 M}}{\sqrt{M}}
\]
for all $k$.
\end{restatable}
\noindent The proof is an immediate application of \ref{G2:1_rip} and Hoeffding's inequality and is omitted. We now return to Lemma \ref{lemma:exp_detail}. We know $\|\sum_{k\in\OZ} \bfd_k\bfd_k^T\| \leq 9/8$ by \ref{G2:s_rip}; thus by above lemma, the first term is bounded by $C\sqrt{\alpha\log^3 M}/\sqrt{M}$ with probability at least $1-2KM^{-\alpha}$. For the second term, we have from Lemma \ref{lemma:nuisance}, \ref{G1:DDT_norm}, \ref{G2:s_rip}, and the triangle inequality that
\[
\left\|\sum_{k=1}^K \inner{\tbfy_0^k}{\bfd_k}^2\bfd_k\bfd_k^T\right\|_2 = \max_k \inner{\tbfy_0^k}{\bfd_k}^2 \sup_{\|\bfz\|_2=1}\bfz^T\bfD\bfD^T\bfz \leq \frac{C\alpha Ks\log^3 M}{M^2}
\]
with the same probability.

This proves line \ref{exp_detail_eqn} from Lemma \ref{lemma:exp_detail}; we now incorporate the covariance projection step. We need to show that Frobenius projection of $E[\HSig_0]$ onto $E\bfy\bfy^T = \frac{s}{K}\bfD\bfD^T$ removes the $\bfD\bfD^T$ term in equation \ref{decomp_eqn} while contributing only a negligible factor elsewhere. Since projection is scale-invariant, it suffices to prove this for projection onto $\bfD\bfD^T$ in place of $E\bfy\bfy^T$. We prove the following lemma:

\begin{restatable}[Projection Error Bound]{lma}{covProjLemma}\label{lemma:cov_proj_lemma}
For any $\bfD \in \mcG$,
\[
\left\|\frac{K}{s}\proj{\bfD\bfD^T}{E\HSig_0} - \left(\frac{s}{K}\sum_{k =1}^K \inner{\bfy_0}{\bfd_k}^2\right)\bfD\bfD^T\right\|_2 \leq \frac{Cs}{M} + \frac{CKs^2}{M^3}.
\]
\end{restatable}
This lemma follows from $\mcG$, the triangle inequality, and the definition of Frobenius projection. A detailed proof can be found in Appendix \ref{appdx:cov_proj}. This lemma implies that the error in Lemma \ref{lemma:exp_detail}, Equation \ref{exp_detail_cov} is dominated by the error from Equation \ref{exp_detail_eqn}, confirming Lemma \ref{lemma:exp_detail}.

\subsection{Bounding Estimation Error}
We now bound the resulting error from observing only the finite sample $\bfY$ consisting of $N$ random vectors, which will result in the following lemma:
\begin{restatable}[Estimation Error Bound]{lma}{empCnvgc}\label{lemma:empCnvgc}
Recall that $\HSig_0^{\text{proj}} = \HSig_0 - \proj{\HSig}{\HSig_0}$. Then for any constant $\alpha \geq 1$, $\bfD \in \mcG$, and $N \geq C\alpha\max\left\{\frac{s^{10}\log M}{M^6},\frac{K^2s^4\log^3 M}{M^3}\right\}$,
\[
\left\|\frac{K}{s}\HSig_0^{\text{proj}} - \left(\frac{s-1}{K-1}\bfv_0\bfv_0^T + \sum_{k \in \Omega_0}\bfd_k\bfd_k^T\right)\right\|_2 \leq \frac{C\alpha Ks\log^3 M}{M^2} + \frac{CKs^2\sqrt{\alpha\log^3 M}}{M^{3/2}\sqrt{N}} + \frac{Cs^5\sqrt{\alpha\log M}}{M^3\sqrt{N}}.
\]
with probability at least $1-2N\exp(-s^2/K)-(2K+4)M^{-\alpha}$.
\end{restatable}

We will employ the following theorem (\citet{Vershynin}, theorem 5.6.1) on covariance estimation, versions of which are well-known in the literature:

\begin{thm}[\citet{Vershynin}, Theorem 5.6.1, General Covariance Estimation (Tail Bound)]\label{cov_thm}
Let $\bfz$ be a random vector in $\R^M$. Assume that, for some $\kappa \geq 1$,
\[
\|\bfz\|_2 \leq \kappa\sqrt{E[\|\bfz\|_2^2]}
\]
almost surely. Then, for every positive integer $N$, $\{\bfz_i\}_{i=1}^N$ i.i.d. copies of $\bfz$, and $t \geq 0$, we have:
\[
\left\|E\bfz\bfz^T - \frac{1}{N}\sum_{i=1}^N\bfz_i\bfz_i^T\right\|_2 \leq C{\|E\bfz\bfz^T\|_2}\left(\sqrt{\frac{\kappa^2 M(\log M+t)}{N}} + \frac{\kappa^2 M(\log M+t)}{N}\right)
\]
with probability at least $1-2\exp(-t)$.
\end{thm}

We will apply this theorem to the correlation-weighted  random vectors $\bfz = \inner{\bfy_0}{\bfy}\bfy$ given $\bfD$, for which we will derive the following bounds on the expectation of $\|\bfz\|_2$:
\begin{restatable}{lma}{expNorm}\label{lemma:expNorm}
Suppose that $\bfD \in \mcG$. Then the following bounds hold:
\[
E\|\inner{\bfy_0}{\bfy}\bfy\|_2^2 \leq \frac{Cs^3}{M}
\]
\[
\|E\inner{\bfy_0}{\bfy}^2\bfy\bfy^T\|_2 = \|E\HSig_0\|_2 \leq \frac{Cs^3}{M^2}.
\]
\[
\|E\inner{\bfy_0}{\bfy}^2\bfy\bfy^T\|_F = \|E[\HSig_0]\|_F \leq \frac{Cs^3}{M^{3/2}}.
\]
\end{restatable}
The result follows from $\mcG$ and similar computations to \ref{lemma:decompLemma}; details are in Appendix \ref{appdx:expNorm}.

With the use of a truncation trick, we can use these bounds to prove the following bound on the deviation of $\HSig_0 = \frac{1}{N}\sum_{i=1}^N \inner{\bfy_0}{\bfy_i}\bfy_i\bfy_i^T$ from its expectation:
\begin{restatable}[Correlation-Weighted Covariance Estimation Bound]{lma}{covLemma}\label{cov_lemma}
For any $\bfD \in \mcG$ and constant $\alpha \geq 1$,
\[
\|\HSig_0 - E\HSig_0\|_2 \leq \frac{Cs^3\sqrt{\alpha\log^3 M} }{M^{3/2}\sqrt{N}} + \frac{Cs^3\alpha\log^3 M}{MN}
\]
and 
\[
\|\HSig_0 - E[\HSig_0]\|_F \leq \frac{Cs^3\sqrt{\alpha\log^3 M} }{M\sqrt{N}} + \frac{Cs^3\alpha\log^3 M}{\sqrt{M}N}
\]
with probability at least $1-2N\exp(-s^2/K)-2M^{-\alpha}$.
\end{restatable}
This lemma follows by applying Theorem \ref{cov_thm} to the truncated random vector $\bfz = \bbi_{\omega}\inner{\bfy_0}{\bfy}\bfy$ where $\omega$ is the event that $\|\inner{\bfy_0}{\bfy}\bfy\|_2\leq Cs^{3/2}\log M/\sqrt{M}$; details are in Appendix \ref{appdx:cov_lemma}.

We can derive an analogous bound for the unweighted sample covariance:
\begin{lma}\label{sig_d_corr}
Recall that $\HSig = \frac{1}{N}\bfY\bfY^T = \frac{1}{N}\sum_{i=1}^N\bfy_i\bfy_i^T$ with $E\HSig = E\bfy\bfy^T = \frac{s}{K}\bfD\bfD^T$. Then for $\bfD \in \mcG$ and any constant $\alpha \geq 1$,
\[
\left\|\frac{1}{N}\bfY\bfY^T - E\bfy\bfy^T\right\|_2 =\left\|\HSig - \frac{s}{K}\bfD\bfD^T\right\|_2\leq \frac{Cs\sqrt{\alpha\log M}}{\sqrt{M}\sqrt{N}} + \frac{Cs\alpha\log M}{\sqrt{N}}
\]
and
\[
\left\|\frac{1}{N}\bfY\bfY^T - E\bfy\bfy^T\right\|_F = \left\|\HSig - \frac{s}{K}\bfD\bfD^T\right\|_F \leq \frac{Cs\sqrt{\alpha\log M}}{\sqrt{N}} + \frac{Cs\alpha\log M\sqrt{M}}{\sqrt{N}}
\]
with probability at most $1-2M^{-\alpha}$.
\end{lma}
This is proven by immediate application of Theorem \ref{cov_thm}, so a detailed proof is omitted.

Having proven the above bounds, proving convergence of $\HSig_0^\text{proj}$ amounts to an extended computation with the triangle inequality, which can be found in Appendix \ref{appdx:empCnvgc}. Combined with Lemma \ref{lemma:exp_detail}, these computations complete the proof of Lemma \ref{lemma:empCnvgc}.

\subsection{Bounding Subspace Error}
It remains to demonstrate that the first $s$ eigenvectors of $\HSig_0^\text{proj}$ span a subspace close to $\mcS_0 = \vspan\{\bfd_k\}_{k\in\Omega_0}$. It is easy to see that this holds in an asymptotic sense: by \ref{G1:DDT_norm}, we know $\bfD\bfD^T\bfy_0/\|\bfD\bfD^T\bfy_0\|_2 \to \bfy_0/\|\bfy_0\|_2$, while $\sum_{k\in\Omega_0}\bfd_k\bfd_k^T$ will be close to an identity matrix on the subspace $\mcS_0$. However, acquiring quantitative bounds on the subspace distance is more challenging and requires some technical machinery. 

We begin by introducing the following notation: given a matrix $\bfA \in \R^{M\times M}$, let $\lambda_1(\bfA),\ldots,\lambda_M(\bfA)$ be the eigenvalues of $\bfA$ in descending order. Similarly, for $m \in \{1,\ldots, M\}$ let $\mcS_{m}(\bfA)$ denote the subspace spanned by the eigenvectors of $\bfA$ corresponding to $\lambda_1(\bfA), \ldots, \lambda_m(\bfA)$.

We will prove the following result:
\begin{thm}\label{thm:subspace}
Let $\mcS_0 = \vspan\{\bfd_k\}_{k\in\Omega_0}$ and le $\alpha \geq 1$ be constant. As long as
\[
\tilde{\varepsilon} = \frac{\alpha Ks\log^3 M}{M^2} + \frac{Ks^2\sqrt{\alpha\log^3 M}}{M^{3/2}\sqrt{N}} + \frac{s^5\sqrt{\alpha\log M}}{M^3\sqrt{N}}.
\]
is sufficiently small, then with with probability at least $1-2N\exp(-s^2/K)-(2K+4)M^{-\alpha}$,
\[
\mcD\left(\mcS_0, \mcS_s\left(\HSig_0^\text{proj}\right)\right) \leq \frac{C\sqrt{M}}{\sqrt{K}} + C\tilde{\varepsilon}.
\]
\end{thm}

We can recognize the $\tilde{\varepsilon}$ in this bound as the bound from Lemma \ref{lemma:empCnvgc}, while the $\sqrt{M}/\sqrt{K}$ term arises from the deviation of $\bfv_0$ from a multiple of $\bfy_0$. To prove the theorem, we can write the result of Lemma \ref{lemma:empCnvgc} as
\[
\HSig_0^{\text{proj}} = \frac{s}{K}\left(\bfv_0\bfv_0^T + \sum_{k \in \Omega_0}\bfd_k\bfd_k^T + \mcE\right)
\]
where $\mcE$ is a symmetric matrix with norm bounded by $C\tilde{\epsi}$ by Lemma \ref{lemma:empCnvgc}. We will subsequently ignore the outer factor of $s/K$ as this has no effect on the matrix's invariant subspaces.

We claim that the spectral properties of $\frac{s-1}{K-1}\bfv_0\bfv_0^T + \sum_{k\in\Omega_0}\bfd_k\bfd_k^T$ are essentially the same as those of $\bfy_0\bfy_0^T + \sum_{k\in\Omega_0}\bfd_k\bfd_k^T$, which clearly has $s$-leading subspace $\mcS_0$. If this holds, it then follows from a routine application of the Davis-Kahan theorem on $\mcE$ that $\HSig_0^{\text{proj}}$ have lead $s$ eigenvectors approximately spanning $\mcS_0$ with an additional error proportional to $\tilde{\varepsilon}$ as in the theorem. Therefore, proving the theorem reduces to proving the following lemma:

\begin{restatable}{lma}{subspaceLemma}\label{subspace_lemma}
Let
\[
\bfB = \frac{s-1}{K-1}\bfv_0\bfv_0^T + \sum_{k\in\Omega_0}\bfd_k\bfd_k^T
\]
Then for all $i \leq s$, $\lambda_i(\bfB) > c$ and
for all $i > s$, $\lambda_i(\bfB) \leq C\sqrt{M}/\sqrt{K}$. Moreover, recalling that $\mcS_s(\bfB)$ is the subspace spanned by the leading $s$ eigenvectors of $\bfB$, we have
\[
\mcD(\mcS_0, \mcS_s\left(\bfB\right)) \leq \frac{C\sqrt{M}}{\sqrt{K}}.
\]
\end{restatable}

The key ingredients in our proof are Weyl's and the Davis-Kahan theorem; the proof amounts to applying these to our particular situation and is left for Appendix \ref{appdx:subspace_lemma}. This lemma completes the proof of Theorem \ref{thm:subspace_recovery}.

\subsection{Guarantees for subspace intersection}

We now prove Theorem \ref{int_lemma}, which states that the intersection step with close enough estimated subspaces $\hmcS_1, \ldots, \hmcS_\ell$ accurately approximates the intersection of true subspaces $\mcS_1,\ldots,\mcS_\ell$. The result follows from the following lemma:
\begin{restatable}{lma}{separationLemma}\label{separationLemma}
Let $\bfD \in \mcG$ and suppose $\Omega_1,\Omega_2$ are disjoint, at-most-$s$-element subsets of $\{1,2,\ldots,K\}$. Let $\mcS_1 = \vspan\{\bfd_k\}_{k\in\Omega_1}$ and $\mcS_j = \vspan\{\bfd_m\}_{m\in\Omega_2}$. Then \[
\mcD(\mcS_1, \mcS_2) \geq 1 - 3\delta_{2s}^2 \geq 1 - \frac{Cs\log^2 M}{M}.
\]
\end{restatable}
The proof of this lemma is a straightforward application of the $2s$-RIP of $\bfD$ and can be found in Appendix \ref{appdx:separationLemma}. It follows from Lemma \ref{separationLemma} that for large enough $M$, two subspaces $\mcS_i$ and $\mcS_j$ contain vectors closer than a constant threshold if and only if they share support ($|\Omega_i \cap \Omega_j| \geq 1$). Since this result holds with high probability, this holds for all pairs $i,j$ simultaneously. Since this holds for pairwise intersections, the analogous result holds for $\ell$-wise intersections as well.

We can now prove Theorem \ref{intersection_thm}. By  Lemma \ref{separationLemma} above, $\mcD(\hmcS_i,\hmcS_j)$ will be small if and only if samples $\bfy_i$ and $\bfy_j$ share support. Theorem \ref{thm:subspace_recovery} then provides the quantitative bound of order $\varepsilon$ for a single intersection; since there are at most $\ell$ intersections, the triangle inequality bounds the total error at $C\ell\varepsilon \leq C\epsi \log M$. This completes the proof of Theorem \ref{intersection_thm}.

\subsection{Guarantees for Oracle Refinement}

In this section, we prove Theorem \ref{thm:oracle_ref}, which shows how we can boost an initial estimate obtained by SSDL to a much stronger one by the oracle refinement procedure. The key to this step is to show that we can accurately detect the support of each sample $\bfy_i$ by looking at the norm of the projections $\|\bfP_{\hmcS_i}\hbfd_k\|_2$, with large norm indicating a nonzero coefficient. In the following results and proofs, we assume that the previous step returned an estimate of $\bfd_k$ as opposed to $-\bfd_k$, but the result applies equally in the opposite case.

To do this, we will employ the following lemma with proof in the appendix:

\begin{restatable}{lma}{subspaceNorm}
\label{lemma:subspace_norm}
Let $\mcS_1, \mcS_2$ be subspaces of $\R^M$ with associated projection matrices $\bfP_2$ and $\bfP_2$. Then we have
\[
\|\bfP_{1} - \bfP_{2}\|_2 \leq 2\mcD(\mcS_1, \mcS_2)
\]
\end{restatable}
With this lemma in hand, we can now show that thresholding the projections accurately detects the support of $\bfy_i$:

\begin{lma}[Oracle support]\label{lemma:or_support}
Let
\[
\max_{i=1,\ldots,N} \mcD(\mcS_i, \hmcS_i) \leq 1/8 \hspace{10pt} \text{ and } \hspace{10pt}\max_{k=1,\ldots,K} \|\hbfd_k - \bfd_k\|_2 \leq 1/8.
\]
Then for all $i$ and $k$, $\|\bfP_{\hmcS_i}\hbfd_k\| > 1/2$ if and only if $k \in \Omega_i$.
\end{lma}

\begin{proof}
We write:
\[
\bfP_{\hmcS_i}\hbfd_k = \bfP_{\mcS_i}\bfd_k + (\bfP_{\hmcS_i} -\bfP_{\mcS_i})\bfd_k + \bfP_{\hmcS_i}(\hbfd_k - \bfd_k).
\]
Since $\|\bfP_{\hmcS_i}\|_2 = 1$, 
\[
\|\bfP_{\hmcS_i}(\hbfd_k - \bfd_k)\|_2 \leq \|\hbfd_k - \bfd_k\|_2 \leq 1/8.
\]
Next, we note that by Lemma \ref{lemma:subspace_norm}, $\|(\bfP_{\hmcS_i} -\bfP_{\mcS_i})\bfd_k\|_2 \leq 1/4$. Accordingly,
\[
\|\bfP_{\mcS_i}\bfd_k\|_2 - 3\varepsilon \leq \|\bfP_{\hmcS_i}\hbfd_k\|_2 \leq \|\bfP_{\mcS_i}\bfd_k\|_2 + 3/8.
\]
As long as $k \in \Omega_i$, $\|\bfP_{\mcS_i}\bfd_k\|_2 = 1$, meaning $\|\bfP_{\hmcS_i}\hbfd_k\|_2 > 1/2$. By contrast, when $k \notin \Omega_i$, \ref{G2:s_rip} guarantees that $\|\bfP_{\mcS_i}\bfd_k\|_2 \leq \frac{C\sqrt{s}\log M}{\sqrt{M}} < 1/8$, meaning $\|\bfP_{\hmcS_i}\hbfd_k\|_2 < 1/2$ for $M$ sufficiently large. This completes the proof.
\end{proof}

This lemma indicates that as long as our initial estimate $\hbfD$ is sufficiently accurate---which it will be in high enough dimensions by Theorem \ref{thm:ssdl}---we can recover the support of each sample exactly. It remains to show that, with this information, we can construct a superior estimator to that recovered by SSDL.  We proceed similarly to the proof of Theorem \ref{thm:subspace_recovery}, beginning with the following lemmas on the expectation $\bfV_k = E[\bfy\bfy^T|k\in\Omega]$:

\begin{restatable}{lma}{oracleCovExp}\label{lemma:oracle_cov_exp}
For all $k$,
\[
\bfV_k = E[\bfy\bfy^T|k\in\Omega] = \left(1-\frac{s-1}{K-1}\right)\bfd_k\bfd_k^T + \frac{s-1}{K-1}\bfD\bfD^T.
\]
\end{restatable}

This lemma is a computation and its proof is deferred to the appendix. We now bound the error which results from applying covariance projection to this matrix:

\begin{lma}
For all $k$,
\[
\|\proj{\bfD\bfD^T}{\bfd_k\bfd_k^T}\|_2 \leq \frac{CK}{M^{3/2}}.
\]
\end{lma}
\begin{proof}
Again, without loss of generality set $k = 1$. We have
\[
\|\proj{\bfD\bfD^T}{\bfd_1\bfd_1^T}\|_2 = \frac{\inner{\bfd_1\bfd_1^T}{\bfD\bfD^T}_F}{\|\bfD\bfD^T\|_F}\|\bfD\bfD^T\|_2 \leq \frac{\inner{\bfd_1\bfd_1^T}{\bfD\bfD^T}_F}{\sqrt{M}}
\]
where the last inequality uses \ref{C1:D_norm} and \ref{C2:D_norm_F}. Further, we can write
\[D
\inner{\bfd_1\bfd_1^T}{\bfD\bfD^T} = \sum_{k=1}^K\inner{\bfd_1}{\bfd_k}^2 = 1 + \sum_{k=1}^K \inner{\bfd_1}{\bfd_k}^2 = 1 + \bfd_1^T\bfD\bfD^T\bfd_1 \leq \frac{CK}{M}
\]
by \ref{G1:DDT_norm}. This completes the proof.
\end{proof}

Together with Lemma \ref{lemma:oracle_cov_exp} (oracle support), this proves that in expectation, covariance projection yields a nearly rank-one matrix in the direction of $\bfd_k\bfd_k^T$:
\[
\left\|E\tbfV_k - \proj{\bfD\bfD^T}{E\tbfV_k} - \left(1-\frac{s-1}{K-1}\right)\bfd_k\bfd_k^T\right\|_2 \leq \frac{CK}{M^{3/2}}
\]
It remains to show that we can estimate the quantity $E[\bfy\bfy^T|k\in\Omega]$ with sufficiently many samples. From Lemma \ref{lemma:or_support} (oracle support), when the initial error $\varepsilon$ is small enough, the oracle projection step exactly identifies the support of each sample $\bfy_i$, so we can proceed as though the support were known exactly. Let $N_k$ denote the number of samples containing dictionary element $k$ and let $\tbfV_k$ be the sampled version of $\bfV_k = E[\bfy\bfy^T|k\in\Omega]$, that is:
\[
\tbfV_k = \frac{1}{N_k}\sum_{k \in \Omega_i} \bfy_i\bfy_i^T.
\]
By Lemma \ref{lemma:or_support}, oracle refinement recovers support exactly, so this definition of $\tbfV_k$ is equivalent to that which can be computed directly using the sets $\tO_i$ in place of $\Omega_i$.
\begin{lma}\label{lemma:oracle_emp}
Let $\alpha$ be any constant. Then for all $k$,
\[
\|\tbfV_k - \bfV_k\|_2 \leq \frac{C\sqrt{\alpha KM \log M}}{\sqrt{Ns}} + \frac{CK\alpha M \log M}{Ns}
\]
with probability at least $1-K\exp(-Ns/(10K)) - KM^{-\alpha}$.
\end{lma}

\begin{proof}
Condition on $N_k$. By the $s$-RIP of $\bfD$ (\ref{G2:s_rip}), we know that $\|\bfy\|^2_2 = \|\bfD\bfx\|^2_2 \leq C\|\bfx\|^2_2 = Cs$; accordingly we can apply theorem \ref{cov_thm} to the random vector $\bfy$ conditional on $\bfd_k\in\Omega$ with $\kappa = C$. The result follows from the choice of $t = \alpha \log M$ and applying Lemma \ref{lemma:nk} to undo the conditioning on $N_k$.
\end{proof}

The proof of Theorem \ref{thm:oracle_ref} now amounts to stitching together the above bounds. Details can be found in Appendix \ref{appdx:oracle_ref}.


\subsection{Guarantees for Oracle Averaging}

In this section, we prove Theorem \ref{thm:oracle_avg}. We begin with the following result, which shows that columns of $\tbfD$ are precise enough to detect not only the support of columns of $\bfX$ but also the sign of the nonzero elements of $\bfX$:
\begin{lma}[Oracle signs]
Suppose $\|\tbfd_k - \bfd_k\|_2 \leq 3/(8\sqrt{s})$. Then for any $i$ such that $k \in \Omega_i$, $\sign{\inner{\tbfd_k}{\bfy_i}} = \sign{x_{ik}}$.
\end{lma}
We note that this result assumes that the previous steps recovered an estimate of $\bfd_k$ rather than of $-\bfd_k$; the latter case is identical except signs will be reversed.
\begin{proof}
Without loss of generality, assume $x_{ik} = 1$. We can write
\[
\inner{\tbfd_k}{\bfy_i} = 1 + \inner{\bfd_k}{\bfy_i - \bfd_k} + \inner{\bfd_k - \tbfd_k}{\bfy_i}
\]
By the Cauchy-Schwarz inequality and \ref{C3:y_norm}, $|\inner{\bfd_k - \tbfd_k}{\bfy_i}| \leq \|\bfd_k-\tbfd_k\|_2\|\bfy_i\|_2 \leq 2\varepsilon < 6/8$. Likewise by \ref{G2:s_rip}, $|\inner{\bfd_k}{\bfy_i - \bfd_k}| \leq \delta_{2s} \leq C\sqrt{s}\log M/\sqrt{M} < 1/8$. The result follows.
\end{proof}

This result guarantees that with high probability, Algorithm \ref{algo:oracle_averaging} recovers the signs of nonzero entries of $\bfX$ exactly. Accordingly, the averaging process of Algorithm \ref{algo:oracle_averaging} amounts to simple averaging of all samples $\bfy_i$ with support containing $k$. We have the following result:
\begin{lma}\label{lemma:avg}
Let $\bfD \in \R^{M\times K}$ be a dictionary satisfying $\mcG$, and for each $k=1,\ldots,K$ let
\[
\bbfd_k = \frac{1}{N_k}\sum_{k\in\Omega_i}x_{ik}\bfy_i.
\]
Then for each $\alpha \leq Ns/(2K)$,
\[
\|\bbfd_k - \bfd_k\|_2 \leq \left(\frac{3}{2}+3\sqrt{\alpha\log M}\right)\frac{\sqrt{K}}{\sqrt{N}}
\]
for all $k$ with probability at least $1-K\exp(-Ns/(10K)) - KM^{-\alpha}$.
\end{lma}

In our proof, we employ the following theorem, due to Gross \citep{vectorBernstein}:
\begin{thm}[Vector Bernstein Inequality \citep{vectorBernstein}]\label{thm:vectorBernstein}
Let $\bfz_1, \ldots, \bfz_n$ be independent, bounded, mean zero random vectors with norm bounded by $\gamma > 0$. Let
\[
V = \sum_{i=1}^n E\|\bfz_i\|_2^2.
\]
Then for all $t \leq V/\gamma$,
\[
\BBP{\left\|\sum_{i=1}^n \bfz_i\right\|_2 \geq \sqrt{V} + t} \leq \exp\left(\frac{-t^2}{4V}\right)
\]
\end{thm}

\noindent We now prove lemma \ref{lemma:avg}.
\begin{proof}
We begin by conditioning on $N_k$ and assuming that we are on the event that Lemma \ref{lemma:nk} holds, which occurs with probability at least $1-K\exp(-Ns/(10K))$. Since the distribution of $\bfX$ is symmetric, without loss of generality we may assume $x_{ik} = 1$ for all $i$ such that $k\in\Omega_i$. We then have
\[
\bbfd_k = \frac{1}{N_k} \sum_{k\in\Omega_i} \bfy_i = \bfd_k + \frac{1}{N_k} \sum_{k\in\Omega_i} (\bfy_i - \bfd_k).
\]
We will now concentrate the sum $\frac{1}{N_k} \sum_{k\in\Omega_i} (\bfy_i - \bfd_k)$. We can write
$\bfy_i - \bfd_k = \bfD\left(\bfx_i - \bfe_k\right)$
where $\bfe_k$ is the $k$-th coordinate vector. Since $k\in\Omega_i$, the vectors $\bfx_i - \bfe_k$ are symmetric and thus mean zero. Moreover, $\bfx_i - \bfe_k$ is $(s-1)$-sparse, and so by the $s$-RIP of $\bfD$ (\ref{G2:s_rip}), $\sqrt{s}/2 \leq \|\bfD(\bfx_i - \bfe_k)\|_2 \leq 3\sqrt{s}/2$ for all $i$ with $k\in\Omega_i$. It follows that
\[
\frac{N_k s}{4} \leq \sum_{k\in\Omega_i} E\|\bfy_i - \bfd_k\|_2^2 \leq \frac{9N_k s}{4}.
\]
Accordingly, we can apply Theorem \ref{thm:vectorBernstein} to deduce that
\[
\left\|\sum_{k\in\Omega_i} (\bfy_i - \bfd_k)\right\|_2 \leq \frac{3\sqrt{N_k s}}{2} + t
\]
for any $t \leq N_k \sqrt{s}/4$ with probability at least $1-\exp\left(-t^2/(9N_k s)\right)$. Choosing $t = u\sqrt{N_k s}$ for any $u \leq \sqrt{N_k}/4$, we have
\[
\left\|\frac{1}{N_k}\sum_{k\in\Omega_i} (\bfy_i - \bfd_k)\right\|_2 \leq \left(\frac{3}{2} + u\right)\frac{\sqrt{s}}{\sqrt{N_k}}
\]
with probability at least $1-\exp(-u^2/9)$. We undo the conditioning on $N_k$ by the bound $N_k \geq cNs/K$ from Lemma \ref{lemma:nk}; thus choosing $u = 3\sqrt{\alpha\log M}$ gives
\[
\left\|\frac{1}{N_k}\sum_{k\in\Omega_i} (\bfy_i - \bfd_k)\right\|_2 \leq \left(\frac{3}{2} + 3\sqrt{\alpha \log M}\right)\frac{\sqrt{K}}{\sqrt{N}}
\]
with probability at least $1- M^{-\alpha}$. A union bound over $k$ completes the proof.
\end{proof}

This result completes the proof of Theorem \ref{master_thm}.

\section{Numerical Simulations}\label{simulations}
All code used in these simulations, including an open-source Python implementation of \textbf{SPORADIC}, is publicly available at \url{https://github.com/sew347/SPORADIC_DL}. In this section, we supplement our theoretical results with numerical simulations. We consider two metrics. The first is  convergence in angle: how close is $|\inner{\hbfd}{\bfd}|$ to 1. This measure can always be converted to the error in $L^2$ norm by the identity $\|\hbfd-\bfd\|_2^2 = 2 - 2\inner{\hbfd}{\bfd}$ (up to possible differences in sign). The second main metric is the proportion of \textit{false recoveries}: a \textit{false recovery} occurs when subspace intersection either returns a vector when a dictionary element does not exist in the intersection of true subspaces, or when subspace intersection fails to return a vector when one is in the intersection of true subspaces.

\subsection{Experiment 1: Maximum Tolerated Sparsity}
To test our theoretical hypotheses, in Figure \ref{max_sparsity}, we show the maximum sparsity that can be tolerated by SSDL while remaining above a certain accuracy threshold. Specifically, this figure shows the highest sparsity for which our test returns average angular accuracy above $0.95$ with false recovery proportion below $0.08$. Tests were run on the first 50 subspaces over 5 dictionaries in each dimension. To ensure a nontrivial number of overlaps, the support of each of the first 50 samples was seeded so that $1 \in \Omega_1 \cap \Omega_2, 2 \in \Omega_3 \cap \Omega_4, \ldots, 25 \in \Omega_{49} \cap\Omega_{50}$. In this test $K = 2M$, and $N$ was pegged to $N=30000$ for $M=500$ then allowed to grow as different powers of $M$. The results confirm the findings of Theorem \ref{master_thm}: sparsity growth is linear when $N\sim M^4$, but slower than linear when $N\sim M^3$ or $N\sim M^2$.

\begin{figure*}
\centering
\includegraphics[width=0.7\linewidth]{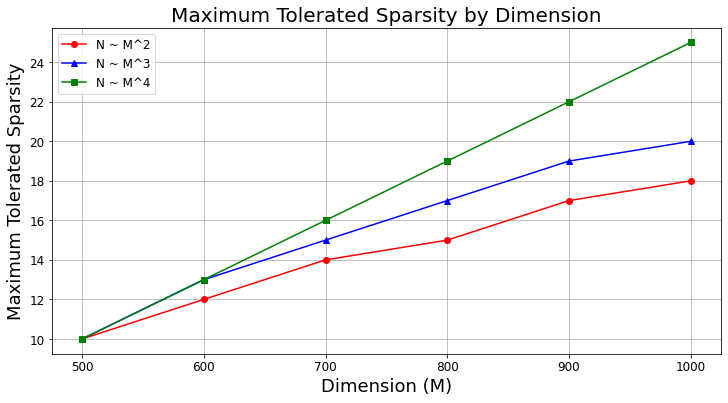}
\caption{\textbf{Maximum tolerated sparsity.} As predicted by Theorem \ref{master_thm}, $N \sim M^4$ allows for linear sparsity growth in $M$, while fewer samples result in sublinear growth. }
\label{max_sparsity}
\end{figure*}

\subsection{Experiment 2: Substep Error by Number of Samples}

In our second experiment, we compare performance in terms of $L^2$ error for the three methods that comprise \textbf{SPORADIC}. Figure \ref{fig:substep_err} shows error for multiple sample sizes for the three substeps of \textbf{SPORADIC} (SSDL, oracle refinement, and oracle averaging) along with ``true averaging," which is the oracle averaging procedure applied with the support and signs known exactly rather than estimated. This experiment was run in dimension $M = 300$ with a dictionary of size $K=600$. Subspace intersection was performed for the first 500 subspaces, meaning only a subset of the dictionary was recovered. Errors were measured only for dictionary elements that were correctly recovered; false recoveries were ignored in this experiment.

We see that given enough samples, the error of oracle averaging tends to that of true averaging, confirming the oracle properties discussed in the main text. Somewhat surprisingly, though, oracle refinement enjoys a slightly lower error than true averaging for every sample size $N$. This suggests one of two possibilities: either, due to hardware limitations, the experiment was not run on large enough $N$ for the advantages of oracle averaging to manifest; or oracle refinement actually enjoys better performance than is guaranteed by theorem \ref{master_thm}, making the oracle averaging step unnecessary. We leave the resolution of this question for future work.

\begin{figure*}
\centering
\includegraphics[width=0.7\linewidth]{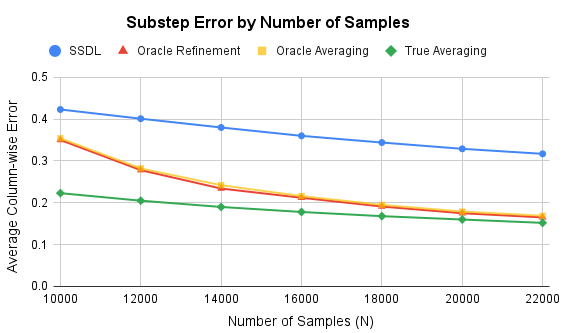}
\caption{\textbf{Column-wise error by number of samples.} As predicted by Theorem \ref{master_thm}, with enough samples oracle averaging tends to the same performance as averaging with known support and signs. Unexpectedly, oracle refinement slightly outperforms oracle averaging regardless of number of samples.}
\label{fig:substep_err}
\end{figure*}

\section{Conclusion}
We introduced \textbf{SPORADIC} as an efficient method for recovering dictionaries from high-dimensional samples even in the linear sparsity regime up to log factors. In this regime, \textbf{SPORADIC} achieves arbitrarily small errors in polynomial time given enough samples, improving on the best-known provable alternatives. Reproducible numerical simulations validate these results, with the surprising caveat that oracle refinement appears to work better than expected.

Our initial research on \textbf{SPORADIC} suggests several avenues for future research. First, we suspect it is possible to reduce the sample complexity from approximately $M^4$ to closer to $M^3$ by replacing the covariance projection step with a more sophisticated method for controlling the dominant term in \ref{decomp_eqn}. This is because the worst term in the error in Theorem \ref{thm:subspace_recovery} comes from estimating the covariance matrix in the Frobenius rather than in the operator norm, which is known to have worse sample complexity ($N \sim M^2$) than estimation in the operator norm ($N \sim M$). We are also interested in investigating to what degree the uniformity assumption in the generation of support sets $\Omega_j$ can be relaxed, to allow for more general sampling distributions. Lastly, as our method involves a fourth-order statistic, it bears some similarity to the recently introduced $L^4$-based dictionary learning methods; future work will study these connections in greater detail.


\acks{This paper is an expanded version of the earlier conference paper \citep{White23}. Both authors were partially supported by grants NSF DMS-1813943 and AFOSR FA9550-20-1-0026. The authors are free from any conflicts of interest.}



\appendix

\section{Technical Proofs}

In this section, we include proofs of technical lemmas which we omitted from the main text.

\subsection{Proof of Lemma \ref{G0_lemma}}\label{appdx:g0}
\Glemma*

We prove the two items in Definition \ref{good_event} as a separate lemma. In what follows, we use the term ``with high probability'' to mean a sequence of events whose complements occur with probability tending to zero faster than any inverse polynomial in $M$. A convenient fact is that, by a union bound, any polynomial number of high-probability events likewise occurs with high probability. Under this definition, Lemma \ref{G0_lemma} amounts to proving that for large enough $M$, $\bfD$ is in $\mcG(C)$ with high probability for some sufficiently large absolute constant $C$.

We will make use of the relationship $\bfd_k = \bfw_k/\|\bfw_k\|_2$ where $\bfw_k$ are i.i.d. random vectors with independent $N(0,1)$ entries. Expressed in matrix form, we have $\bfD = \bfW\bfN$ where $\bfW$ is a matrix with columns $\bfw_k$ and $\bfN$ is a diagonal matrix with $k$-th diagonal entry $1/\|\bfw_k\|_2$. We will make use of the following lemma:

\begin{lma}\label{w_lemma}
Let $\bfw$ be an $N(0,\bfI)$ random vector. Then with high probability, $\sqrt{M} - \log M \leq \|\bfw\|_2 \leq \sqrt{M} + \log M$.
\end{lma}
\begin{proof}
By definition, entries of $\bfw$, denoted $w_m$ for $m=1,\ldots,M$, are i.i.d. normal random variables with variance $1$. It follows that $\|\bfw\|_2^2 = \sum_{m=1}^M w_m^2$ will be distributed as a chi-squared random variable with $M$ degrees of freedom. Such a variable obeys the following concentration inequalities \citep{LaurentMassart}:
\begin{equation}\label{chi_sq_conc}
\begin{aligned}
\BBP{\|\bfw\|_2^2 \geq M + \sqrt{Mt} + 2t} \leq e^{-t}
\\
\BBP{\|\bfw\|_2^2 \leq M - \sqrt{Mt}} \leq e^{-t}
\end{aligned}
\end{equation}
which for $t \leq M$ yields the symmetric bound:
\begin{equation}\label{chi_sq_sym}
\BBP{\left|\|\bfw\|_2^2 - M\right| \geq 3\sqrt{Mt}} \leq 2e^{-t}
\end{equation}

To convert this to a bound on $\|\bfw\|_2$, we use the fact that for any $z, \gamma \geq 0$, $|z - 1| \geq \gamma$ implies $|z^2 - 1| \geq \max\{\gamma, \gamma^2\}$. Accordingly,
\[
\BBP{\left|\frac{1}{\sqrt{M}}\|\bfw\|_2 - 1\right| \geq \delta} \leq \BBP{\left|\frac{1}{M}\|\bfw\|^2_2 - 1\right| \geq \max\{\delta, \delta^2\}}
\]
Setting $\delta = \log M /\sqrt{M}$ and applying \ref{chi_sq_sym} with $t = (\log^2 M)/9 \leq M$, we have:
\[
\BBP{|\|\bfw\|_2 - \sqrt{M}| \leq \log M} \leq \BBP{|\|\bfw\|_2^2 - M| \leq \sqrt{M}\log M} \leq e^{-(\log^2 M)/9}
\]
which proves the result.
\end{proof}

\begin{lma}[\ref{G1:DDT_norm}]
For $\bfD$ distributed as in Lemma \ref{G0_lemma}, there exists absolute constant $C$ such that $\left\|\bfD\bfD^T - \frac{K}{M}\bfI\right\|_2 \leq \frac{C\sqrt{K}}{\sqrt{M}}$ with high probability.
\end{lma}
\begin{proof} 
We use the normal identity $\bfD = \bfW\bfN$ to write:
\[
\left\|\bfD\bfD^T - \frac{K}{M}\bfI\right\|_2 = \left\|\bfW\bfN^2\bfW^T - \frac{K}{M}\bfI\right\|_2.
\]
By the triangle inequality,
\begin{equation}\label{W_tri_ineq}
\left\|\bfW\bfN^2\bfW^T - \frac{K}{M}\bfI\right\|_2 \leq \left\|\frac{1}{M}\bfW\bfW^T - \frac{K}{M}\bfI\right\|_2 + \left\|\bfW\bfN^2\bfW^T - \frac{1}{M}\bfW\bfW^T\right\|_2
\end{equation}
By Lemma \ref{w_lemma}, with high probability $M - \sqrt{M}\log M \leq \|\bfw_k\|_2^2 \leq M +\sqrt{M}\log M$ for all $k$. Therefore:
\[
\left\|\bfW\bfN^2\bfW^T - \frac{1}{M}\bfW\bfW^T\right\|_2 \leq \|\bfW\|_2^2\left\|\bfN^2 - \frac{1}{M}\bfI\right\|_2 \leq \frac{C\|\bfW\|_2^2}{M^{3/2}}
\]
We now bound $\left\|\bfW\bfW^T - K\bfI\right\|_2$. We employ an $\epsi$-net argument \citep[e.g.][]{Vershynin}. We let $\mcM$ be a $1/4$-net on the unit sphere; that is, $\mcM$ is a finite set such that every point on the unit sphere is at most Euclidean distance $1/4$ from a point in $\mcM$. It is a fact of $\varepsilon$-nets that for $\epsi = 1/4$, we can choose $\mcM$ such that $|\mcM| \leq 9^M$. Moreover:
\[
\left\|\bfW\bfW^T - K\bfI\right\|_2 = \sup_{\|\bfz\|=1}\bfz^T\left(\bfW\bfW^T - K\bfI\right)\bfz \leq 2\max_{\bfz \in \mcM} \bfz^T\left(\bfW\bfW^T - K\bfI\right)\bfz
\]
We now fix $\bfz \in \mcM$. We have:
\[
\bfz^T\left(\bfW\bfW^T - K\bfI\right)\bfz = \sum_{k=1}^K \inner{\bfz}{\bfw_k}^2 - K.
\]
By rotational invariance, $\inner{\bfz}{\bfw_k} \sim N(0,1)$, so we can concentrate this sum using the chi-squared concentration inequalities \ref{chi_sq_conc}. Choosing $t = CM$, we have:
\[
\left|\left(\sum_{k=1}^K \inner{\bfz}{\bfw_k}^2\right)- K\right| \leq \sqrt{CKM} + 2CM
\]
with probability at least $1 - e^{-CM}$. Unfixing $\mcM$ by a union bound, this holds for all $\bfz \in \mcM$ with probability at least $1-9^Me^{-CM}$, which represents high probability for sufficiently large $C$. Thus with high probability, changing $C$ if necessary,
\[
\left\|\bfW\bfW^T - K\bfI\right\|_2 \leq {C\sqrt{KM}}
\]
Plugging back into \ref{W_tri_ineq}, noting that the above implies $\|\bfW\|_2^2 \leq CK$, we conclude:
\[
\left\|\bfD\bfD^T - \frac{K}{M}\bfI\right\|_2 \leq \frac{C\sqrt{K}}{\sqrt{M}} + \frac{CK}{M^{3/2}}
\]
with high probability. Since $K/M^{3/2} \leq \sqrt{K}/\sqrt{M}$ by definition of $\mcG$, this completes the proof.
\end{proof}

\noindent We now prove \ref{G2:1_rip}:

\begin{lma}[\ref{G2:1_rip}]
For $\bfD$ distributed as in Lemma \ref{G0_lemma}, there exists absolute constant $C$ such that for all $k \neq m$, $|\inner{\bfd_k}{\bfd_m}| \leq C\log M/\sqrt{M}$.
\end{lma}
\begin{proof}
Since $k\neq m$, $\bfd_k$ and $\bfd_m$ are independent. Then by rotational invariance, we can treat $\bfd_m$ as a fixed vector, say the first coordinate vector $\bfe_1$: $\inner{\bfd_k}{\bfd_m} \sim \inner{\bfd_k}{\bfe_1}$. Writing $\bfd_k = \bfw_k/\|\bfw_k\|_2$ for $\bfw_k \sim N(0,\bfI)$, we have $\inner{\bfd_k}{\bfe_1} = \frac{1}{\|\bfw\|_2}\inner{\bfw_k}{\bfe_1} \sim \frac{1}{\|\bfw\|_2} \times N(0,1)$. It is known (see \citep{Vershynin}, proposition 2.1.2) that a normally distributed random variable $Z_\sigma \sim N(0,\sigma^2)$ obeys the concentration inequality:
\begin{equation}\label{gauss_conc}
\BBP{Z_\sigma \geq t} \leq \frac{\sigma}{\sqrt{2\pi} t}\exp\left(\frac{-t^2}{2
\sigma^2}\right)
\end{equation}
Applying this to $\inner{\bfw_k}{\bfe_1}$, we have $\left|\inner{\bfw_k}{\bfe_1}\right| \leq C\log M$ with high probability. By Lemma \ref{w_lemma}, we know that $\|\bfw_k\|_2 \geq c/\sqrt{M}$ with high probability, so we conclude the result.
\end{proof}

Lastly, we prove \ref{G2:s_rip}:

\begin{lma}[\ref{G2:s_rip}]
For $\bfD$ distributed as in Lemma \ref{G0_lemma}, there exists absolute constant $C$ such that $\bfD$ has $2s$-restricted isometry constant satisfying $\delta_{2s} \leq \frac{C\sqrt{s}\log M}{\sqrt{M}}$ with high probability.
\end{lma}
\begin{proof}
We can write $\sum_{k\in\OZ}\bfd_k\bfd_k^T = \bfD_\OZ\bfD_\OZ^T$ where $\bfD_\OZ$ is the submatrix of $\bfD$ with columns indexed by $\OZ$. Since the nonzero eigenvalues of $\bfD_\OZ\bfD_\OZ^T$ are the same as those of $\bfD_\OZ^T\bfD_\OZ$, we may prove the bound for the latter matrix. By Lemma \ref{w_lemma}, it suffices to prove the result for the matrix $\frac{1}{M}\bfW_\OZ^T\bfW_\OZ$.

As in the proof of \ref{G1:DDT_norm}, we employ an $\varepsilon$-net argument. Let $\mcM$ be a $1/4$-net of the unit sphere in $\R^s$, which can be chosen with at most $9^s$ elements. For fixed $\bfz \in \mcM$, we have:
\[
\bfz^T\left(\frac{1}{M}\bfW_\OZ^T\bfW_\OZ - \bfI\right)\bfz
\]
Since $\bfW$ has i.i.d. $N(0,1)$ entries, $\bfz^T\bfW_\OZ^T\bfW_\OZ\bfz$ will be distributed as a chi-squared random variable with $M$ degrees of freedom. Accordingly, by \ref{chi_sq_conc}, we have that
\[
|\bfz^T\bfW_\OZ^T\bfW_\OZ\bfz - M| \leq \sqrt{CMs} + 2Cs
\]
with probability at least $1-e^{-Cs}$. Unfixing $\bfz$ by a union bound, this holds for all $\bfz \in \mcM$ with probability at least $1-9^s e^{-Cs}$, which is high probability for large enough $C$. We conclude that with high probability,
\[
\left\|\frac{1}{M}\bfW_\OZ^T\bfW_\OZ - \bfI\right\|_2  \leq \frac{C\sqrt{s}}{\sqrt{M}}
\]
which implies the result.
\end{proof}

\subsection{Proof of Lemma \ref{lemma:decompLemma}}\label{appdx:decomp}
\decompLemma*

\begin{proof}
As all expectations in this lemma are conditional on $\bfy_0$, we will not write this explicitly in the proof. We can expand: 
\[
E\inner{\bfy_0}{\bfy}^2\bfy\bfy^T = E\sum_{k_1\in \Omega}\sum_{k_2\in \Omega}\sum_{k_3\in \Omega}\sum_{k_4\in \Omega}x_{k_1}x_{k_2}x_{k_3}x_{k_4}\inner{\bfy_0}{\bfd_{k_1}}\inner{\bfy_0}{\bfd_{k_2}}\bfd_{k_3}\bfd_{k_4}^T
\]
Since $Ex_{ik} = 0$ for all $i,k$, terms in the above expectation will be nonzero only when the indices are paired. Accordingly:
\[
E\inner{\bfy_0}{\bfy}^2\bfy\bfy^T = E\sum_{k \in \Omega}\sum_{m \in \Omega}x_k^2x_m^2\inner{\bfy_0}{\bfd_k}^2\bfd_m\bfd_m^T + E\sum_{k \in \Omega}\sum_{m \in \Omega, m \neq k}x_k^2x_m^2\inner{\bfy_0}{\bfd_k}\inner{\bfy_0}{\bfd_m}\bfd_k\bfd_m^T
\]
\[
= E\sum_{k =1}^K\sum_{m = 1}^K\bbi_{\{k,m\} \subseteq \Omega}\inner{\bfy_0}{\bfd_k}^2\bfd_m\bfd_m^T + E\sum_{k =1}^K\sum_{m \neq k}^K \bbi_{\{k,m\} \subseteq \Omega}\inner{\bfy_0}{\bfd_k}\inner{\bfy_0}{\bfd_m}\bfd_k\bfd_m^T
\]
\[
= \sum_{k =1}^K\sum_{m = 1}^K\BBP{\{k,m\} \subseteq \Omega}\inner{\bfy_0}{\bfd_k}^2\bfd_m\bfd_m^T + \sum_{k =1}^K\sum_{m \neq k}^K \BBP{\{k,m\} \subseteq \Omega}\inner{\bfy_0}{\bfd_k}\inner{\bfy_0}{\bfd_m}\bfd_k\bfd_m^T
\]
\small\[
= \frac{s}{K}\sum_{k=1}^K\inner{\bfy_0}{\bfd_k}^2\bfd_k\bfd_k^T + \frac{s(s-1)}{K(K-1)}\sum_{k =1}^K\sum_{m \neq k}^K\inner{\bfy_0}{\bfd_k}^2\bfd_m\bfd_m^T + \frac{s(s-1)}{K(K-1)}\sum_{k = 1}^K\sum_{m \neq k}^K\inner{\bfy_0}{\bfd_k}\inner{\bfy_0}{\bfd_m}\bfd_k\bfd_m^T
\]\normalsize
Next, we complete the square by transferring part of the first term into the other two:
\begin{multline}\label{decomp_partial}
E\inner{\bfy_0}{\bfy}^2\bfy\bfy^T = 
\left(\frac{s}{K}-\frac{2s(s-1)}{K(K-1)}\right)\left(\sum_{k=1}^K\inner{\bfy_0}{\bfd_k}^2\bfd_k\bfd_k^T\right) +
\\ +\frac{s(s-1)}{K(K-1)}\left(\sum_{k =1}^K \inner{\bfy_0}{\bfd_k}^2\right)\left(\sum_{k =1 }^K\bfd_k\bfd_k^T\right) 
+ \frac{s(s-1)}{K(K-1)}\left(\sum_{k=1}^K \inner{\bfy_0}{\bfd_k}\bfd_k\right)\left(\sum_{k=1}^K \inner{\bfy_0}{\bfd_k}\bfd_k\right)^T
\end{multline}
Lastly, we note that for $k \in \Omega_0$, $\inner{\bfy_0}{\bfd_k} = 1 + \inner{\bfy_0 - \bfd_k}{\bfd_k} = 1 + \inner{\tbfy_0^k}{\bfd_k}$, while for $k > s$, $\tbfy_0^k = \bfy_0$. Accordingly we can substitute:
\[
\sum_{k=1}^K\inner{\bfy_0}{\bfd_k}^2\bfd_k\bfd_k^T = \sum_{k\in\Omega_0} \bfd_k\bfd_k^T 
+ 2\sum_{k\in\Omega_0}\inner{\tbfy_0^k}{\bfd_k}\bfd_k\bfd_k^T
+ \sum_{k=1}^K \inner{\tbfy_0^k}{\bfd_k}^2\bfd_k\bfd_k^T
\]
The result follows by making these substitutions in \ref{decomp_partial}, factoring out $s/K$, and noting that $\sum_{k=1}^K\bfd_k\bfd_k^T = \bfD\bfD^T$ and $\bfv_0 = \bfD\bfD^T\bfy_0 = \sum_{k=1}^K \inner{\bfy_0}{\bfd_k}\bfd_k$ by definition.
\end{proof}




\subsection{Proof of Lemma \ref{lemma:cov_proj_lemma}}\label{appdx:cov_proj}
\covProjLemma*

The proof will employ the following lemma, which states that Frobenius projection with $\bfD$ will not increase the 2-norm of a matrix by more than a constant factor.
\begin{lma}\label{lemma:fro_l2_lemma}
Suppose $\mcG$ holds. Then for any matrix $\bfA \in \R^{M\times M}$,
\[
\left\|\frac{\inner{\bfA}{\bfD\bfD^T}_F}{\|\bfD\bfD^T\|^2_F}\bfD\bfD^T\right\|_2 \leq C\|\bfA\|_2.
\]
\end{lma}
\begin{proof}
By the Cauchy-Schwarz inequality, \ref{G1:DDT_norm}, and \ref{C2:D_norm_F},
\[
\left\|\frac{\inner{\bfA}{\bfD\bfD^T}_F}{\|\bfD\bfD^T\|^2_F}\bfD\bfD^T\right\|_2 \leq \frac{\|\bfA\|_F\|\bfD\bfD^T\|_2}{\|\bfD\bfD^T\|_F} \leq \frac{C\|\bfA\|_F}{\sqrt{M}}.
\]
The result follows from the fact that $\|\bfA\|_F \leq \sqrt{M}\|\bfA\|_2$.
\end{proof}

We now prove Lemma \ref{lemma:cov_proj_lemma}:

\begin{proof}[Proof of Lemma \ref{lemma:cov_proj_lemma}]
From \ref{decomp_eqn}, we have:
\small\begin{multline}
\frac{K}{s}E\inner{\bfy_0}{\bfy}^2\bfy\bfy^T = \frac{s-1}{K-1}\bfv_0\bfv_0^T + \left(1 - \frac{2(s-1)}{K-1}\right)\sum_{k\in\Omega_0}\bfd_k\bfd_k^T + 
\\ 
+\left(1 - \frac{2(s-1)}{K-1}\right)\left(2\sum_{k\in\Omega_0} \inner{\tbfy_0^k}{\bfd_k} \bfd_k\bfd_k^T
+ \sum_{k=1}\inner{\tbfy_0^k}{\bfd_k}^2\bfd_k\bfd_k^T\right) 
+ \left(\frac{s-1}{K-1}\sum_{k =1}^K \inner{\bfy_0}{\bfd_k}^2\right)\bfD\bfD^T
\end{multline}\normalsize
The fifth term will be removed entirely by projection. Since the third and fourth terms are known to be small from the main text, Lemma \ref{lemma:fro_l2_lemma} indicates we only need to bound the projection of the first and second terms. 

We first bound inner products $\inner{\sum_{k\in\Omega_0} \bfd_k\bfd_k^T}{\bfD\bfD^T}_F$ and $\inner{\bfv_0\bfv_0^T}{\bfD\bfD^T}_F$. We consider each term individually, beginning with the first. We write:
\begin{equation}\label{eqn:1st_inner_prod}
\inner{\sum_{k\in\Omega_0}\bfd_k\bfd_k^T}{\bfD\bfD^T}_F = \sum_{k\in\Omega_0}\sum_{m = 1}^K \inner{\bfd_k}{\bfd_m}^2 = \sum_{k\in\Omega_0} \bfd_k^T\bfD\bfD^T\bfd_k \leq s \times \|\bfD\bfD^T\|_2 \leq Ks/M
\end{equation}
by \ref{C1:D_norm}.

Lastly, we bound $\inner{\frac{s-1}{K-1}\bfv_0\bfv_0^T}{\bfD\bfD^T}_F$. We have:
\[
\left|\inner{\frac{s-1}{K-1}\bfv_0\bfv_0^T}{\bfD\bfD^T}_F\right| \leq \frac{s}{K}\left|\sum_{k=1}^K \inner{\bfv_0\bfv_0^T}{\bfd_k\bfd_k^T}_F\right| = \frac{s}{K}\sum_{k=1}^K \inner{\bfv_0}{\bfd_k}^2.
\]
Noting that $\bfv_0 = \bfD\bfD^T\bfy_0$, we know that
\begin{equation}\label{eqn:2nd_inner_prod}
\frac{s}{K}\sum_{k=1}^K \inner{\bfv_0}{\bfd_k}^2 = \frac{s}{K}\bfv_0^T (\bfD\bfD^T)\bfv_0 = \frac{s}{K}\bfy_0^T (\bfD\bfD^T)^3\bfy_0 \leq \frac{s}{K}\|\bfD\bfD^T\|_2^3\|\bfy_0\|_2^2 \leq \frac{CK^2s^2}{M^3}
\end{equation}
by \ref{G1:DDT_norm} and \ref{G2:s_rip}. Further, by \ref{G1:DDT_norm} and \ref{C2:D_norm_F}, we have
$\|\bfD\bfD^T\|_2/\|\bfD\bfD^T\|^2_F \leq \frac{C}{K}$. We combine this with \ref{eqn:1st_inner_prod} and \ref{eqn:2nd_inner_prod} to conclude that
\[
\left\|\frac{\inner{\sum_{k\in\Omega_0}\bfd_k\bfd_k^T}{\bfD\bfD^T}_F+\inner{\frac{s-1}{K-1}\bfv_0\bfv_0^T}{\bfD\bfD^T}_F}{\|\bfD\bfD^T\|_F^2}\bfD\bfD^T\right\|_2 \leq \frac{Cs}{M} + \frac{CKs^2}{M^3}.
\]
as desired.
\end{proof}


\subsection{Proof of Lemma \ref{lemma:expNorm}}\label{appdx:expNorm}
\expNorm*

\begin{proof}
We repeat the computations of Lemma \ref{lemma:decompLemma}, replacing outer products with inner products. This yields:
\small\[
E\|\inner{\bfy_0}{\bfy}\bfy\|_2^2 = \left(\frac{s}{K}-\frac{2s(s-1)}{K(K-1)}\right)\sum_{k=1}^K\inner{\bfy_0}{\bfd_k}^2 + \frac{s(s-1)}{K(K-1)}\left(K\sum_{k =1}^K \inner{\bfy_0}{\bfd_k}^2 + \left\|\sum_{k=1}^K \inner{\bfy_0}{\bfd_k}\bfd_k\right\|_2^2\right)
\]\normalsize
\[
= \left(\frac{s(s-1)}{(K-1)} + \frac{s}{K} - \frac{2s(s-1)}{K(K-1)}\right)\sum_{k=1}^K\inner{\bfy_0}{\bfd_k}^2 + \frac{s(s-1)}{K(K-1)}\left\|\bfD\bfD^T\bfy_0\right\|_2^2.
\]
We know by \ref{G1:DDT_norm} and \ref{G2:s_rip} that
\[
\sum_{k=1}^K \inner{\bfy_0}{\bfd_k}^2 = \bfy_0^T\bfD\bfD^T\bfy_0 \leq \frac{CKs}{M}
\]
and that
\[
\left\|\bfD\bfD^T\bfy_0\right\|_2^2 = \bfy_0^T(\bfD\bfD^T)^2\bfy_0 \leq \frac{CK^2s}{M^2}.
\]
We conclude that $E\|\inner{\bfy_0}{\bfy}\bfy\|_2^2 \leq \frac{Cs^3}{M}$.

We proceed to bound the covariance $E\inner{\bfy_0}{\bfy}^2\bfy\bfy^T = E\HSig_0$. From Lemma \ref{lemma:exp_detail}, we can infer that $E\inner{\bfy_0}{\bfy}^2\bfy\bfy^T$ satisfies:
\[
E\HSig_0 = \frac{s}{K}\left(\frac{s-1}{K-1}\bfv_0\bfv_0^T + \sum_{k \in \Omega_0}\bfd_k\bfd_k^T + \left(\frac{s}{K}\sum_{k =1}^K \inner{\bfy_0}{\bfd_k}^2\right)\bfD\bfD^T + \mcE\right)
\]
where $\mcE$ has negligible norm, as does $\sum_{k\in\Omega_0}\bfd_k\bfd_k^T$ by \ref{G2:s_rip}. By \ref{G1:DDT_norm} and \ref{G2:s_rip}, we know that
\[
\frac{s}{K}\left\|\frac{s-1}{K-1}\bfv_0\bfv_0^T\right\|_2 \leq \frac{s^2}{K^2}\|\bfv_0\|_2^2 = \frac{s^2}{K^2}\times \bfy_0^T(\bfD\bfD^T)^2\bfy_0 \leq \frac{s^2}{K^2}\times \frac{CK^2s}{M^2} = \frac{Cs^3}{M^2}
\]
and likewise that
\[
\left(\frac{s^2}{K^2}\sum_{k =1}^K \inner{\bfy_0}{\bfd_k}^2\right)\bfD\bfD^T = \frac{s^2}{K^2}\times \bfy_0^T\bfD\bfD^T\bfy_0 \times \|\bfD\bfD\|_2 \leq \frac{s^2}{K^2}\times\frac{CK^2s}{M^2} = \frac{Cs^3}{M^2}.
\]
This proves the bound on $\|E\HSig_0\|_2$; the bound on the Frobenius norm follows immediately from the fact that $\|E\HSig_0\|_F \leq \sqrt{M}\|E\HSig_0\|_2$.
\end{proof}

\subsection{Proof of Lemma \ref{cov_lemma}}\label{appdx:cov_lemma}
We first require the following lemma:
\begin{lma}\label{lemma:corr_weight_norm}
With probability at least $1-2\exp(-s^2/K)$,
\[
\|\inner{\bfy_0}{\bfy_i}\bfy_i\|_2 \leq C\log M \sqrt{E[\|\inner{\bfy_0}{\bfy_i}\bfy_i\|_2^2]} = \frac{Cs^{3/2}\log M}{\sqrt{M}}
\]
\end{lma}

\begin{proof}
By \ref{C3:y_norm}, we know that
\[
\|\inner{\bfy_0}{\bfy_i}\bfy_i\|_2 \leq |\inner{\bfy_0}{\bfy_i}|\|\bfy_i\|_2 \leq C|\inner{\bfy_0}{\bfy_i}|\sqrt{s}.
\]
We now control the term $|\inner{\bfy_0}{\bfy_i}|$. Denote by $\bfx_0^*$ the restriction of $\bfx_0$ to the shared support of $\bfx_0$ and $\bfx$, and likewise define $\bfx_i^*$. Then, noting that $\inner{\bfy_0}{\bfy_i} = \inner{\bfD\bfx_0}{\bfD\bfx_i}$,
\[
|\inner{\bfy_0}{\bfy_i}| \leq |\inner{\bfD\bfx_0^*}{\bfD\bfx_i^*} + \inner{\bfD\bfx_0^*}{\bfD(\bfx_i - \bfx_i^*)} + \inner{\bfD((\bfx_0 - \bfx_0^*)}{\bfD\bfx_i^*} + \inner{\bfD(\bfx_0^*}{\bfD\bfx_i^*}| 
\]
\[
\leq |\inner{\bfD(\bfx_0^*}{\bfD\bfx_i^*}|+3s\delta_{2s} \leq |\inner{\bfD\bfx_0^*}{\bfD\bfx_i^*}| +\frac{Cs^{3/2}\log M}{\sqrt{M}}
\]
since the other terms are by definition disjoint. It remains to bound the overlap term. We have:
\[
|\inner{\bfD\bfx_0^*}{\bfD\bfx_i^*}| = \sum_{k\in\Omega_0 \cap \Omega_i} x_{0k}x_{ik}
\]
Condition on $|\Omega_0 \cap \Omega_i|$. By Hoeffding's inequality, for any $t \geq 0$,
\begin{equation}\label{eqn:cond_hoeffding}
|\inner{\bfD\bfx_0^*}{\bfD\bfx_i^*}| < \sqrt{2t|\Omega_0 \cap \Omega_i|}
\end{equation}
with probability at least $1-\exp(-t)$. To undo the conditioning on $|\Omega_0 \cap \Omega_i|$, we note that $E|\Omega_0 \cap \Omega_i| = s^2/K$. We then concentrate this term using a Chernoff bound\footnote{Strictly speaking, as elements in $\Omega_i$ are not chosen independently, a Chernoff bound cannot be applied directly. However, the negative correlation of sampling with replacement guarantees better concentration properties than if elements of $\Omega_i$ were chosen independently with probability $s/K$; see \citet{Dubhashi_Ranjan_1996} for details.} \citep[e.g.,][Theorem 2.3.1]{Vershynin} we see that $|\Omega_0 \cap \Omega_i| \leq es^2/K$ with probability at least $1-\exp(-s^2/K)$. Thus plugging back into \ref{eqn:cond_hoeffding} with $t = s^2/K$, we have
\[
|\inner{\bfD\bfx_0^*}{\bfD\bfx_i^*}| < \frac{\sqrt{2\log(2) et}s}{\sqrt{K}} \leq \frac{\sqrt{2\log(2) e}s^2}{K}
\]
with probability at least $1-2\exp(-s^2/K)$. Since $K \gg M$ and $s/M , 1$, $s^2/K \ll s^{3/2}/\sqrt{M}$. Thus we conclude that
\[
|\inner{\bfy_0}{\bfy_i}| \leq \frac{Cs^{3/2}\log M}{\sqrt{M}}
\]
with probability at least $1-2\exp(-s^2/K)$ as desired.
\end{proof}

We now prove lemma \ref{cov_lemma} itself:
\covLemma*

\begin{proof}
We are interested in estimating the true covariance matrix of the random vector $\inner{\bfy_0}{\bfy}\bfy$ by its sample covariance $\frac{1}{N}\sum_{i=1}^N \inner{\bfy_0}{\bfy_i}^2\bfy_i\bfy_i^T$. By Lemma \ref{lemma:corr_weight_norm}, we know that with probability at least $1-2\exp(-s^2/K)$,
\[
\|\inner{\bfy_0}{\bfy_i}\bfy_i\|_2 \leq C\log M \sqrt{E[\|\inner{\bfy_0}{\bfy}\bfy\|_2^2]}
\]
We would like to apply Theorem \ref{cov_thm} with $\kappa = C\log M$, but since this is only with high probability, not almost surely, we cannot apply it directly. 

Instead, let $\omega$ be the event that $\|\inner{\bfy_0}{\bfy}\bfy\|^2_2 \leq Cs^{3/2}\log M/\sqrt{M}$ and consider the truncated random vector $\bfz = \bbi_{\omega}\inner{\bfy_0}{\bfy}\bfy$. By Lemma \ref{lemma:corr_weight_norm}, $\omega$ occurs with probability $1-2\exp(-s^2/K)$, and therefore $\bfz = \bfy$ with the same probability. Since $\bfy_0$ and $\bfy$ are each sum of $s$ unit vectors, it is easy to see that $\|\inner{\bfy_0}{\bfy}\bfy_0\|_2^2 \leq s^6$. Therefore:
\[
E\|\inner{\bfy_0}{\bfy}\bfy\|_2^2 - (1-\BBP{\omega})s^6 \leq E\|\bfz\|_2^2 \leq E\|\inner{\bfy_0}{\bfy}\bfy\|_2^2
\]
Since $\omega$ occurs with this probability, $|E\|\bfz\|_2^2 - E\|\inner{\bfy_0}{\bfy}\bfy\|_2^2| \leq s^6\exp(-s^2/K)$, which vanishes faster than any polynomial in $M$. Therefore we may safely substitute $E\|\bfz\|_2^2 = E\|\inner{\bfy_0}{\bfy}\bfy\|_2^2$ in the following bounds. With this substitution, we also have that $\|\bfz\|_2 \leq C\log M \sqrt{E\|\bfz\|_2^2}$. 

We now consider the sample $\{\bfy_i\}_{i=1}^N$ consisting of $N$ i.i.d. copies  of $\bfy$, and define $\bfz_i = \bbi_{\omega_i}\inner{\bfy_0}{\bfy_i}\bfy_i$ where $\omega_i$ is the event that $\|\inner{\bfy_0}{\bfy_i}\bfy_i\|_2^2 \leq Cs^{3/2}\log M/\sqrt{M}$. The $\bfz_i$ are i.i.d. copies of $\bfz$, so we can apply Theorem \ref{cov_thm} to the sample covariance of the random vectors $\bfz_i = \bbi_{\omega_i}\inner{\bfy_0}{\bfy_i}\bfy_i$ with $\kappa = C\log M$, yielding:
\[
\left\|\frac{1}{N}\sum_{i=1}^N \bfz_i\bfz_i^T - E\bfz_i\bfz_i^T\right\|_2 \leq C\|E\HSig_0\|_2\left(\sqrt{\frac{M(\log^3 M+t\log^2 M)}{N}}+\frac{M(\log^3 M+t\log^2 M)}{N}\right)
\]
with probability at least $1- 2\exp(-t)$. Applying Lemma \ref{lemma:expNorm}, we know $\|E\HSig_0\|_2 \leq Cs^3/M^2$, so with the same probability we have:
\[
\left\|\frac{1}{N}\sum_{i=1}^N \bfz_i\bfz_i^T - E\bfz_i\bfz_i^T\right\|_2 
\leq \frac{C\sqrt{t}s^3\log M }{M^{3/2}\sqrt{N}} + \frac{Cts^3\log^2 M}{MN}
\]
as long as $t \geq \log M$. By a union bound, $\bfz_i = \inner{\bfy_0}{\bfy_i}\bfy_i$ for all $i$ with probability at least $1-N\exp(-s^2/K)$ on $\omega$. Therefore, choosing $t = \alpha\log M$ for $\alpha \geq 1$, with probability at least $1-2N\exp(-s^2/K)-2M^{-\alpha}$,
\begin{equation}\label{cov_eqn}
\left\|\frac{1}{N}\sum_{i=1}^N \inner{\bfy_0}{\bfy_i}^2\bfy_i\bfy_i^T - E\bfz_i\bfz_i^T\right\|_2 \leq \frac{Cs^3\sqrt{\alpha\log^3 M} }{M^{3/2}\sqrt{N}} + \frac{Cs^3\alpha\log^3 M}{MN}.
\end{equation}
Since we have already concluded that $\|E\HSig_0 - E\bfz_i\bfz_i^T\|_2$ is small, we conclude the result for the operator 2-norm; the bound for Frobenius norm is an immediate corollary.
\end{proof}

\subsection{Proof of Lemma \ref{lemma:empCnvgc}}\label{appdx:empCnvgc}
\empCnvgc*
\begin{proof}
We begin by invoking Lemmas Lemma \ref{lemma:nuisance}, \ref{cov_lemma}, and \ref{sig_d_corr}, which by a union bound hold simultaneously with probability at least $1-2N\exp(-s^2/K)-(2K+4)M^{-\alpha}$. With $N$ as assumed, for $M$ large enough the first terms of these bounds dominate, yielding
\[
\|\HSig_0 - E\HSig_0\|_2 \leq \frac{Cs^3\sqrt{\alpha\log^3 M}}{M^{3/2}\sqrt{N}}
\]
and
\[
\left\|\HSig - \frac{s}{K}\bfD\bfD^T\right\|_2\leq \frac{Cs\sqrt{\alpha\log M}}{\sqrt{M}\sqrt{N}},
\]
along with corresponding bounds on the Frobenius norm.

We aim to bound the difference between the result of covariance projection with expected versus sample covariance matrices:
\[
\|\HSig_0 - \proj{\HSig}{\HSig_0} - (E\HSig_0 - \proj{\bfD\bfD^T}{E\HSig_0})\|_2.
\]
By the triangle inequality this is bounded by:
\begin{equation}\label{triangle_eqn}
\|\HSig_0 - E\HSig_0\|_2 +\|\proj{\bfD\bfD^T}{E\HSig_0} - \proj{\bfD\bfD^T}{\HSig_0}\|_2 + \|\proj{\bfD\bfD^T}{\HSig_0} - \proj{\HSig}{\HSig_0}\|_2
\end{equation}
From Lemma \ref{cov_lemma}, we know that $\|\HSig_0 - E\HSig_0\|_2 \leq \frac{Cs^3\sqrt{\alpha\log^3 M}}{M^{3/2}\sqrt{N}}$. Moreover, it follows from Lemma \ref{lemma:fro_l2_lemma} that 
\[
\|\proj{\bfD\bfD^T}{E\HSig_0} - \proj{\bfD\bfD^T}{\HSig_0}\|_2 \leq \|\HSig_0 - E\HSig_0\|_2 \leq \frac{Cs^3\sqrt{\alpha\log^3 M}}{M^{3/2}\sqrt{N}}.
\]

\noindent We now turn to the final term in \ref{triangle_eqn}, which can be controlled as follows:
\[
\|\proj{\bfD\bfD^T}{\HSig_0} - \proj{\HSig}{\HSig_0}\|_2 = \left\|\frac{\inner{\HSig_0}{\bfD\bfD^T}_F}{\|\bfD\bfD^T\|_F^2} - \frac{\inner{\HSig_0}{\HSig}_F}{\|\HSig\|_F^2}\right\|_2\|\HSig_0\|_2
\]
\begin{equation}\label{term_three}
=\left|\inner{\HSig_0}{\Sigma_{\bfD} - \frac{\|\bfD\bfD^T\|_F^2}{\|\HSig\|_F^2}\HSig}_F\right|\times \frac{\|\HSig_0\|_2}{\|\bfD\bfD^T\|^2_F}\leq \left\|\bfD\bfD^T - \frac{\|\bfD\bfD^T\|_F^2}{\|\HSig\|_F^2 }\HSig\right\|_2 \times \frac{\|\HSig_0\|_F\|\HSig_0\|_2}{\|\bfD\bfD^T\|_F^2}
\end{equation}
By Lemma \ref{cov_lemma}, since $N$ is large enough by assumption we can substitute $\|E\HSig_0\|_2$ and $\|E\HSig_0\|_F$ for $\|\HSig_0\|_2$ and $\|\HSig_0\|_F$ respectively, up to a constant factor. We know from Lemma \ref{lemma:expNorm}, $\|E\HSig_0\|_F \leq \frac{Cs^3}{M^{3/2}}$ while $\|E\HSig_0\|_2 \leq \frac{Cs^3}{M^2}$. Then by \ref{C2:D_norm_F}:
\[
\frac{\|\HSig_0\|_F\|\HSig_0\|_2}{\|\bfD\bfD^T\|_F^2} \leq \frac{Cs^3 M}{M^{3/2}} \times  \frac{Cs^3 M}{M^2} \times \frac{CM}{K^2} \leq \frac{Cs^6 M}{K^2M^{5/2}}
\]
To bound the first term in equation \ref{term_three}, we use the triangle inequality, which yields:
\begin{equation}\label{inner_bound}
\left\|\bfD\bfD^T - \frac{\|\bfD\bfD^T\|_F^2}{\|\HSig\|_F^2 }\HSig\right\|_2 \leq \frac{K}{s}\left(\left\|\frac{s}{K}\bfD\bfD^T - \HSig\right\|_2 + \left|\frac{\left\|\frac{s}{K}\bfD\bfD^T\right\|^2_F}{\|\HSig\|_F^2} - 1\right|\|\HSig\|_2\right)
\end{equation}
By Lemma \ref{sig_d_corr} and \ref{G1:DDT_norm}, with probability at least $1-2N\exp(-s^2/K) - 4M^{-\alpha}$ we have $\|\HSig\|_2 \leq \frac{s}{K} \times \frac{CK}{M} = \frac{Cs}{M}$. By the same corollary and \ref{C2:D_norm_F}, we have $\|\HSig\|_F^2 \geq \frac{s}{K}\times \frac{cK^2}{M} = \frac{cKs}{M}$. Thus, noting that $|\|a\|-\|b\|| \leq \|a-b\|$ for any norm, we have:
\[
\left|\frac{\left\|\frac{s}{K}\bfD\bfD^T\right\|^2_F}{\|\HSig\|_F^2} - 1\right| = \frac{\left|\left\|\frac{s}{K}\bfD\bfD^T\right\|^2_F-\|\HSig\|_F^2\right|}{\|\HSig\|_F^2} \leq \frac{\|\HSig - \frac{s}{K}\bfD\bfD^T\|_F}{\|\HSig\|^2_F} \leq \frac{Cs\sqrt{\alpha\log M}}{\sqrt{N}} \times \frac{CM}{Ks} \leq \frac{CM\sqrt{\alpha\log M}}{K\sqrt{N}}
\]
with probability at least $1-2N\exp(-s^2/K) - 4M^{-\alpha}$, where in the second to last inequality we again used Lemma \ref{sig_d_corr}. Combining the pieces in \ref{inner_bound}, we have:
\[
\left\|\bfD\bfD^T - \frac{\|\bfD\bfD^T\|_F^2}{\|\HSig\|_F^2 }\HSig\right\|_2 \leq \frac{K}{s}\left(\frac{Cs\sqrt{\alpha \log M}}{\sqrt{M}\sqrt{N}} + \frac{CM\sqrt{\alpha\log M}}{K\sqrt{N}}\times \frac{Cs}{M}\right) \leq \frac{CK\sqrt{\alpha\log M}}{\sqrt{M}\sqrt{N}}
\]
Plugging back into \ref{term_three}, we have:
\[
\|\proj{\bfD\bfD^T}{\HSig_0} - \proj{\HSig}{\HSig_0}\|_2 \leq \frac{CK\sqrt{\alpha\log M}}{\sqrt{M}\sqrt{N}} \times \frac{Cs^6}{K^2M^{5/2}} \leq \frac{Cs^6\sqrt{\alpha \log M}}{KM^3\sqrt{N}}
\]
Combining our estimates for the three terms, we have:
\[
\|\HSig_0 - \proj{\HSig}{\HSig_0} - (E\HSig_0 - \proj{\bfD\bfD^T}{E\HSig_0})\|_2 \leq \frac{Cs^3\sqrt{\alpha\log^3 M}}{M^{3/2}\sqrt{N}} + \frac{Cs^6\sqrt{\alpha\log M}}{KM^3\sqrt{N}}
\]
Factoring out a factor of $\frac{s}{K}$ and plugging in to Lemma \ref{lemma:cov_proj_lemma} completes the proof.
\end{proof}

\subsection{Proof of Lemma \ref{subspace_lemma}}\label{appdx:subspace_lemma}
Our proof will depend on the following variants of Weyl's and the Davis-Kahan Theorems:
\begin{thm}[\cite{Weyl}]
Let $\bfA$ and $\mcE$ be symmetric $M\times M$ matrices, and let $\lambda_i(\cdot)$ represent the $i$-th eigenvalue in descending order. Then for all $i = 1,\ldots,M$:
\[
\lambda_i(\bfA) + \lambda_M(\mcE) \leq \lambda_i(\bfA+\mcE) \leq \lambda_i(\bfA) + \lambda_1(\mcE).
\]
It follows that
\[
|\lambda_i(\bfA) - \lambda_i(\bfA+\mcE)| \leq \|\mcE\|_2.
\]
\end{thm}

\begin{thm}[\cite{DavisKahan}]\label{thm:dkThm}
Let $A = E_0A_0E_0^T+E_1A_1E_1^T$ and $A+\mcE = F_0 \Lambda_0 F_0^T +F_1 \Lambda_1 F_1^T$ be symmetric matrices with $[E_0, E_1]$ and $[F_0,F_1]$ orthogonal. If the eigenvalues of $A_0$ are contained in an interval $(a, b)$, and the eigenvalues of $\Lambda_1$ are excluded from the interval $(a-\delta, b+\delta)$ for some $\delta > 0$, then
\[
\|F_1^TE_0\|_2 \leq \frac{\|F_1^T\mcE E_0\|_2}{\delta} \leq \frac{\|\mcE\|_2}{\delta}
\]
for any unitarily invariant matrix norm $\|\cdot\|$.
\end{thm}
Put in the language of subspace distances, this theorem immediately gives the following corollary:
\begin{corollary}
Let $A = E_0A_0E_0^T+E_1A_1E_1^T$, $A+\mcE = F_0 \Lambda_0 F_0^T +F_1 \Lambda_1 F_1^T$, and $\delta$ be as in Theorem \ref{thm:dkThm}. If $\mcS_A$ and $\mcS_{A+\mcE}$ are the subspaces spanned by columns of $E_0$ and $F_0$, respectively, then
\[
\mcD(\mcS_{A}, \mcS_{A+\mcE}) \leq \frac{\|\mcE\|_2}{\delta}.
\]
\end{corollary}

We now apply Weyl's theorem and the Davis-Kahan theorem to our particular situation via the following lemmas. The constants we derive are likely suboptimal, but adequate for our purposes. We have the following lemma:
\begin{restatable}{lma}{spikeLemma}\label{lemma:spikeLemma}
Let $\bfA$ be a rank-$s$ symmetric positive-semidefinite matrix with nonzero eigenvalues satisfying:
\[
0 < \alpha \leq \lambda_1(\bfA), \ldots, \lambda_s(\bfA) \leq \beta
\]
for some constants $\alpha < \beta$. Let $\bfv \in \mcS$ and $\bfu \in \mcS^\perp$ be vectors such that $\|\bfv\|=\|\bfu\|$ and $\|\bfv+\varepsilon\bfu\|_2=1$. For any $\varepsilon \in (0,\alpha/54)$ and $Z > \max\{2\beta, 1\}$, define the matrix $\bfB$ as:
\[
\bfB = Z(\bfv + \varepsilon\bfu)(\bfv +\varepsilon\bfu)^T + \bfA
\]
Then for all $i > s$, $\lambda_i(\bfB) \leq 24\varepsilon$ and
\[
\mcD(\mcS(\bfA), \mcS(\bfB)) \leq \frac{78\beta\varepsilon}{\alpha^2}.
\]
\end{restatable}

This lemma is mainly a restatement of Weyl's Theorem and the Davis-Kahan Theorem to a specific situation, but its proof is sufficiently long that it is provided separately in Appendix \ref{appdx:spikeLemma}. We apply this lemma to prove Lemma \ref{subspace_lemma}:
\subspaceLemma*
\begin{proof}
We now apply Lemma \ref{lemma:spikeLemma} to $\bfB = \frac{s-1}{K-1}\bfv_0\bfv_0^T + \sum_{k\in\Omega_0}\bfd_k\bfd_k^T$. Since $\sum_{k\in\Omega_0}\bfd_k\bfd_k^T$ is symmetric, rank $s$, and symmetric positive-semidefinite with eigenvalues in $[7/8,9/8]$ by \ref{G2:s_rip}, it is a valid choice for $\bfA$; it remains to designate $\bfv$ and $\bfu$. To this end, we write:
\[
\bfv_0 = \bfz_0 + \bfu_0
\]
where $\bfz_0$ is the component of $\bfv$ lying in $\mcS_0$ and $\bfu_0$ is the component in its orthogonal complement. Since $\bfv_0 = \bfD\bfD^T\bfy_0$, we have
\[
\bfv_0 = \frac{K}{M}\bfy_0 + \left(\bfD\bfD^T - \frac{K}{M}\bfI\right)\bfy_0.
\]
Thus, since $\bfy_0 \in \mcS_0$, applying \ref{G1:DDT_norm} yields:
\[
\|\bfv_0\|_2 \geq \|\bfz_0\|_2 \geq \frac{cK}{M}\|\bfy_0\|_2 \geq \frac{cK\sqrt{s}}{M}
\]
and moreover that $\|\bfu_0\|_2 \leq C\sqrt{Ks}/\sqrt{M}$. We can thus write
\[
\bfv_0 = \|\bfv_0\|_2\left(\frac{\bfz_0}{\|\bfv_0\|} + \frac{\bfu_0}{\|\bfv_0\|}\right) = \|\bfv_0\|_2\left(\frac{\bfz_0}{\|\bfv_0\|} + \left(\frac{\|\bfu_0\|_2}{\|\bfz\|_2}\right)\frac{\|\bfz_0\|\bfu_0}{\|\bfv_0\|_2\|\bfu_0\|}\right).
\]
We can now apply Lemma \ref{lemma:spikeLemma} with $Z = \frac{s-1}{K-1}\|\bfv_0\|_2^2$, $\bfv = \bfz_0/\|\bfv_0\|_2$, $\bfu = \frac{\|\bfz\|_2\bfu_0}{\|\bfv_0\|_2\|\bfu_0\|_2}$, and $\varepsilon = \frac{\|\bfu_0\|_2}{\|\bfz_0\|_2}$, which tells us that
\[
\mcD(\mcS_0, \mcS_s(\bfB)) \leq \frac{C\|\bfu_0\|_2}{\|\bfz_0\|_2} \leq \frac{C\sqrt{M}}{\sqrt{K}}
\]
and that all eigenvalues past the $s$-th are bounded by $C\sqrt{M}/\sqrt{K}$, as desired.
\end{proof}

\subsubsection{Proof of Lemma \ref{lemma:spikeLemma}}\label{appdx:spikeLemma}
Here, we prove Lemma \ref{lemma:spikeLemma} from the previous section. We will use the following lemma:
\begin{lma}\label{basis_lemma}
Let $\bfv$ and $\bfu$ be orthogonal unit vectors, and let $\varepsilon \in (0,1)$. Then there exist matrices $\bfB_1, \bfB_2$ with orthonormal columns spanning the orthogonal complements of $\bfv$ and $\bfv + \varepsilon\bfu$, respectively, such that
\[
\|\bfB_1 - \bfB_2\|_2 \leq 4\varepsilon
\]
\end{lma}
\begin{proof}
Without loss of generality, we can assume that $\bfv = \bfe_1$ and $\bfu = \bfe_2$. We can choose $\bfB_1$ and $\bfB_2$ to be identically equal on the orthogonal complement of the span of $\bfe_1,\bfe_2$. Accordingly, it suffices to prove the result for two-dimensional matrices. Let
\[
\bfB_1 = \begin{pmatrix}
    1 & 0 \\
    0 & 1 \\
\end{pmatrix}, \hspace{10pt} 
\bfB_2 = \frac{1}{\sqrt{1+\varepsilon^2}}\begin{pmatrix}
    1 & -\varepsilon \\
    \varepsilon & 1 \\
\end{pmatrix}.
\]
The result follows from taking the difference $\bfB_1 - \bfB_2$ and applying the inequality $\|\bfB_1 - \bfB_2\|_2 \leq \|\bfB_1-\bfB_2\|_F$.
\end{proof}

We now use this lemma to prove \ref{lemma:spikeLemma}:

\begin{proof}[Proof of Lemma \ref{lemma:spikeLemma}]
Since $\bfB$ has rank at most $s+1$, we may restrict consideration to the $s+1$-dimensional subspace given by the direct sum of $\mcS$ with the span of $\bfu$.  We note that Weyl's theorem ensures that the lead eigenvalue $\lambda$ of $\bfB$ satisfies $Z \leq \lambda \leq Z + \beta$. 

We first show that the lead eigenspaces of $\bfB$ differs from that of $Z(\bfv+\varepsilon\bfu)(\bfv+\varepsilon\bfu)^T$ by at most $\varepsilon/Z$. We first prove this result in the case that $\bfA = \bfI_s$, the first $s$ columns of an identity matrix. Consider the lead eigenvector $\bfB$ corresponding to lead eigenvalue $\lambda$. There exists $\bfw$ orthogonal to $\bfv+\varepsilon\bfu$ such that
\[
\bfB(\bfv + \varepsilon\bfu + \bfw) = Z(\bfv+\varepsilon\bfu) + \bfv + \bfI_s\bfw = \lambda(\bfv+\varepsilon\bfu+\bfw)
\]
If $\bfw = 0$, the result holds, so we assume $\|\bfw\|_2 > 0$. Taking inner products of this equation with $\bfw$, we have:
\[
\inner{\bfv}{\bfw} + \inner{\bfI_s\bfw}{\bfw} = \lambda\|\bfw\|_2^2
\]
Since $0 \leq \inner{\bfI_s\bfw}{\bfw} \leq \|\bfw\|_2^2$, we have further that $\inner{\bfv}{\bfw} \geq (\lambda-1)\|\bfw\|_2^2$. Therefore, since $\inner{\bfv+\varepsilon\bfu}{\bfw} = 0$, we have $\inner{\bfv}{\bfw} = -\varepsilon\inner{\bfu}{\bfw}$, so by the Cauchy-Schwarz inequality,
\[
\varepsilon\|\bfu\|_2\|\bfw\|_ \geq (\lambda-1)\|\bfw\|_2^2 \implies \varepsilon \|\bfu\|_2 \geq (\lambda - 1)\|\bfw\|_2
\]
where division by $\|\bfw\|_2$ is permissible as we assumed $\bfw \neq 0$.
Since $\lambda \geq Z$, we conclude that $\|\bfw\|_2 \leq \varepsilon/(Z-1) \leq 2\varepsilon/Z$.

We now extend this to the $s$-dimensional subspace $\mcS$ by studying the complement of the eigenvectors. Let $\bfE$ be a matrix with columns forming and orthogonal basis of the orthogonal complement of $\bfv+\varepsilon\bfu$, and let $\bfE_{\bfw}$ be the same for $\bfv + \varepsilon\bfu + \bfw$. By Lemma \ref{basis_lemma}, we can choose these such that $\|\bfE_1 - \bfE_2\|_2 \leq 8\varepsilon/Z$. We consider the matrix:
\[
\bfE_{\bfw}^T\bfB\bfE_{\bfw} = \bfE^T\bfB\bfE + (\bfE_w-\bfE)^T\bfB\bfE + \bfE^T\bfB(\bfE_\bfw - \bfE) + (\bfE_\bfw - \bfE)^T\bfB(\bfE_\bfw - \bfE).
\]
Since $\|\bfE - \bfE_\bfw\|_2 \leq 8\varepsilon/Z$, we have that
\[
\|(\bfE_w-\bfE)^T\bfB\bfE + \bfE^T\bfB(\bfE_\bfw - \bfE) + (\bfE_\bfw - \bfE)^T\bfB(\bfE_\bfw - \bfE)\|_2 \leq 27\varepsilon.
\]
as each term is bounded individually by $9\varepsilon$. Since $\bfE$ is a basis of the complement of $\bfv+\varepsilon\bfu$, $\bfE^T\bfB\bfE = \bfE^T\bfA\bfE$, which will have eigenvectors spanning the orthogonal complement of $\mcS_s(\bfA)$ with respect to $\bfv+\varepsilon\bfu$. Since the above matrices are symmetric, Weyl's Theorem tells us that any eigenvalues of $\bfE^T\bfB\bfE$ beyond the first (s-1) will be bounded by $24\varepsilon$. Since $\varepsilon < \alpha/54$, the first $s-1$ eigenvalues of $\bfE^T\bfA\bfE$ are separated from zero by at least $\alpha/2$; thus we can apply the Davis-Kahan theorem to conclude that
\[
\mcD(\mcS_{s-1}(\bfE^T\bfA\bfE), \mcS_{s-1}(\bfE_\bfw^T\bfB\bfE_\bfw) \leq 54\varepsilon/\alpha.
\]

We now repeat this application of Lemma \ref{basis_lemma} with bases for the spans of $\bfv$ and $\bfv+\varepsilon\bfu$. We now use $\bfE_\bfv$ as a basis for the orthogonal complement of $\bfv$, chosen according the lemma, with $\bfE$ representing a different basis if necessary. Since
\[
\bfE^T\bfA\bfE = \bfE_\bfv^T\bfA\bfE_\bfv (\bfE-\bfE_\bfv)^T\bfB\bfE + \bfE^T\bfB(\bfE - \bfE_\bfv) + (\bfE - \bfE_\bfv)^T\bfB(\bfE - \bfE_\bfv),
\]
applying the lemma gives us that $\bfE^T\bfA\bfE \approx \bfE_\bfv^T\bfA\bfE_\bfv$ up to a correction with norm at most $12\varepsilon$. By Weyl's and the Davis-Kahan theorem, then,
\[
\mcD(\mcS_{s-1}(\bfE^T\bfA\bfE), \mcS_{s-1}(\bfE_\bfv^T\bfA\bfE_\bfv) \leq 24\varepsilon/\alpha.
\]
We conclude by the triangle inequality that the subspace spanned by eigenvectors $2$ through $s-1$ of $\bfB$ is distance at most $78\varepsilon/\alpha$ from the orthogonal complement of $\bfv$ in $\mcS_s(\bfA)$. Since we already concluded that the leading eigenvector of $\bfB$ is close to $\bfv$, the result for $\bfA = \bfI_s$ follows.

It remains to extend this result to all rank $s$ positive-semidefinite matrices $\bfA$. We consider the matrix $\mcA$ which equals $\bfA$ on $\mcS_s(\bfA)$ and is the identity on its complement. This matrix will be positive-definite and therefore has a square-root-inverse $\mcA^{-1/2}$ with norm at most $1/\sqrt{\alpha}$. By construction, then, $\mcA^{-1/2}\bfA\mcA^{-1/2} = \bfI_s$ as before. We have:
\[
\bfB = \mcA^{1/2}\left(Z\mcA^{-1/2}(\bfv+\varepsilon\bfu)(\bfv+\varepsilon\bfu)\mcA^{-1/2} + \mcI_s\right)\mcA^{1/2}
\]

Since these matrices remain symmetric, we can then apply the previous result to the matrix inside the parentheses with $\bfv \to \frac{\|\bfv\|_2}{\|\mcA^{-1/2}\bfv\|_2}\mcA^{-1/2}\bfv$ and $Z \to Z\|\mcA^{-1/2}\bfv\|^2 \leq Z/\alpha$, giving an subspace error of at most $\frac{78\varepsilon}{\alpha^2}$. Applying the copies of $\mcA^{1/2}$ outside the parentheses can magnify this by at most a further factor of $\beta$, so we conclude that $\mcD(\mcS_{s}(\bfB), \mcS_{s}(\bfA)) \leq 78\beta\varepsilon/\alpha^2$ for generic $\bfA$ as desired.
\end{proof}

\subsection{Proof of Lemma \ref{separationLemma}}\label{appdx:separationLemma}
\separationLemma*
\begin{proof}
Let $\bfx_1$ and $\bfx_2$ be $s$-sparse vectors with disjoint support $\Omega_1$, $\Omega_2$ respectively. It follows that any unit vector in $\mcS_1$ can be expressed as $\bfD\bfx_1/\|\bfD\bfx_1\|_2$ and likewise for $\mcS_2$. 

By the Pythagorean theorem, for any two unit vectors $\bfz_1$ and $\bfz_2$, $\|\bfz_2-\proj{\bfz_1}{\bfz_2}\|_2 = \sqrt{1-\inner{\bfz_1}{\bfz_2}^2}$. The subspace distance can thus be bounded by 
\[
\mcD(\mcS_1, \mcS_2) \geq \left(1 - \max_{\substack{
\supp{\bfx_1} = \Omega_1 \\
\supp{\bfx_2} = \Omega_2
}}\inner{\frac{\bfD\bfx_1}{\|\bfD\bfx_1\|_2}}{\frac{\bfD\bfx_2}{\|\bfD\bfx_2\|_2}}^2\right)^{1/2}
\]
By the $2s$-RIP of $\bfD$ and Lemma \ref{lemma:rop}, we know that $\|\bfD\bfx_1\|_2 \geq (1-\delta_{2s})\|\bfx_1\|_2$ and $\|\bfD\bfx_2\|_2 \geq (1-\delta_{2s})\|\bfx_2\|_2$, while $|\inner{\bfD\bfx_1}{\bfD\bfx_2}|\leq \delta_{2s}\|\bfx_1\|_2\|\bfx_2\|_2$. It follows that
\[
\inner{\frac{\bfD\bfx_1}{\|\bfD\bfx_1\|_2}}{\frac{\bfD\bfx_2}{\|\bfD\bfx_2\|_2}} \leq \frac{\delta_{2s}}{(1-\delta_{2s})^2} \leq 2\delta_{2s}
\]
since $\delta_{2s} \leq 1/8$. Accordingly, 
$\mcD(\mcS_1,\mcS_2) \geq \sqrt{1-4\delta_{2s}^2} \geq 1-3\delta_{2s}^2$. This completes the proof.


\end{proof}

\subsection{Proof of Theorem \ref{int_lemma}}\label{appdx:int_lemma}

\intLemma*

\begin{proof}
We begin by fixing $k$ and then computing the probability that $k$ is the unique element of intersection for a fixed $\Omega^\ell_j$. We know that $\BBP{k \in \Omega^\ell_j} = (s/K)^\ell$, while the probability that another element is in the intersection is bounded by $K(s/K)^\ell$, so we have:
\[
\BBP{\bigcap_{j\in\mcI_i}\Omega_i = \{k\}} \geq \left(\frac{s}{K}\right)^\ell\left(1 - K\left(\frac{s}{K}\right)^\ell\right) \geq \frac{1}{2}\left(\frac{s}{K}\right)^\ell
\]
We now unfix the set $\mcI$ as follows. We divide $\{1,2,\ldots,N\}$ into disjoint $\ell$-element subsets $\mcI_i$. We define the random variables
$H_i$ to be indicators for the events $\left\{\bigcap_{j\in\mcI_i}\Omega_i = \{k\}\right\}$. Since the sets $\mcI_i$ do not overlap, these are $J/\ell$ independent Bernoulli random variables with success probability at least $(s/K)^\ell/2$. Since $(s/K)^\ell < 1/2K$, it follows that:
\[
\BBP{\sum_{i = 1}^{J/\ell}H_i = 0} \leq \left(1-\frac{1}{2}\left(\frac{s}{K}\right)^\ell\right)^{J/\ell} \leq \left(1-\frac{1}{4K}\right)^{J/\ell} \leq \exp\left(-\frac{J}{4K\ell}\right)
\]
Choosing $J = 4(\alpha+1) K\ell \log K$, we conclude that $k$ is the unique element of intersection for one of the sets with probability at least $1-K^{-\alpha-1}$. Unfixing $k$ by a union bound completes the proof.
\end{proof}

\subsection{Proof of Lemma \ref{lemma:subspace_norm}}
\subspaceNorm*
\begin{proof}
We proceed by bounding the norm of $\|\bfP_1\bfx - \bfP_2\bfx\|_2$ for $\|\bfx\|_2 = 1$. We consider the cases $\bfx \in \mcS_1$ and $\bfx \perp \mcS_1$. For $\bfx \in \mcS_1$, we have
\[
\|\bfP_1\bfx - \bfP_2\bfx\|_2 = \|\bfx - \bfP_2\bfx\|_2 \leq \mcD(\mcS_1,\mcS_2)
\]
by definition of $\mcD$. 

Now suppose that $\bfx \perp \mcS_1$, in which case $\|\bfP_1\bfx-\bfP_2\bfx\|_2 = \|\bfP_2\bfx\|_2$. For brevity of notation, denote $\bfy = \bfP_2\bfx/\|\bfP_2\bfx\|_2$ and set $\mcD(\mcS_1,\mcS_2) = \mu$. By definition, $\|\bfy - \bfP_1\bfy\|_2 \leq \rho$. It follows from the triangle inequality that
\[
\|\bfx-\bfy\| \geq |\|\bfx - \bfP_1\bfy\|_2 + \|\bfy - \bfP_1\bfy\|_2| \geq \sqrt{2-\rho^2} - \rho.
\]
Then applying the identity $\|\bfx-\bfy\|_2^2 = \|\bfx\|_2^2 + \|\bfy\|_2^2 - 2\inner{\bfx}{\bfy}$, we see that
\[
\inner{\bfx}{\bfy} \leq \rho\sqrt{2 - \rho} \leq \sqrt{2}{\rho}.
\]
Since $\bfy = \bfP_2\bfx/\|\bfP_2\bfx\|_2$, we conclude that $\|\bfP_2\bfx\|_2 \leq \sqrt{2}\mcD(\mcS_1,\mcS_2)$. For generic $\bfx \in \R^N$, decomposing $\bfx$ into its $\mcS_1$ component and $\mcS_1^\perp$ components, these errors can be magnified by at most a factor of $\sqrt{2}$, which gives the result.
\end{proof}

\subsection{Proof of Lemma \ref{lemma:oracle_cov_exp}}
\oracleCovExp*
\begin{proof}
Without loss of generality we may assume $k = 1$. Then:
\[
\bfV_k= E\left(x_1\bfd_1 + \sum_{k\in\Omega-\{1\}}x_k\bfd_k\right)\left(x_1\bfd_1 + \sum_{k\in\Omega-\{1\}}x_k\bfd_k\right)^T
\]
Since $Ex_k = 0$, diagonal terms cancel, yielding:
\[
E[\bfy\bfy^T|1\in\Omega] = E\bfd_1\bfd_1^T + \sum_{k\in\Omega-\{1\}}\bfd_k\bfd_k^T = \bfd_1\bfd_1^T + \sum_{k=2}^K \BBP{k\in\Omega|1\in \Omega}\bfd_k\bfd_k^T
\]
\[
= \bfd_1\bfd_1^T + \frac{s-1}{K-1}\sum_{k=2}^K \bfd_k\bfd_k^T = \left(1-\frac{s-1}{K-1}\right)\bfd_1\bfd_1^T + \frac{s-1}{K-1}\bfD\bfD^T
\]
completing the proof.
\end{proof}

\subsection{Proof of Theorem \ref{thm:oracle_ref}}\label{appdx:oracle_ref}

\oracleRef*
\begin{proof}
We assume we are on the event that Lemma \ref{lemma:oracle_emp} holds and that $M$ is great enough that the square root term dominates. We have already noted that 
\[
\left\|E\tbfV_k - \proj{\bfD\bfD^T}{E\tbfV_k} - \left(1-\frac{s-1}{K-1}\right)\bfd_k\bfd_k^T\right\|_2 \leq \frac{CK}{M^{3/2}},
\]
for every $k$, and thus by the triangle inequality we need concern ourselves only with the difference between the estimated and expected terms,
\[
\left\|E\tbfV_k - \proj{\bfD\bfD^T}{E\tbfV_k} - \left(\tbfV_k - \proj{\HSig}{\tbfV_k}\right)\right\|_2.
\]
By the triangle inequality, this is bounded by
\begin{equation}\label{eqn:triangle_2}
\|\tbfV_k - E\tbfV_k\|_2 +\|\proj{\bfD\bfD^T}{E\tbfV_k - \tbfV_k}\|_2 + \|\proj{\bfD\bfD^T}{\tbfV_k} - \proj{\HSig}{\tbfV_k}\|_2.
\end{equation}
By Lemma \ref{lemma:oracle_emp}, we know that for eery $k$, $\|\bfV_k - E\bfV_k\|_2 \leq C\sqrt{\alpha KM\log M}/\sqrt{Ns}$. Moreover, by Lemma \ref{lemma:fro_l2_lemma}, 
\[
\|\proj{\bfD\bfD^T}{E\tbfV_k - \tbfV_k}\|_2 \leq \|E\tbfV_k - \tbfV_k\|_2 \leq \frac{C\sqrt{\alpha KM\log M}}{\sqrt{Ns}},
\]
so it remains only to bound the third term of \ref{eqn:triangle_2}. 

We expand:
\[
\|\proj{\bfD\bfD^T}{\tbfV_k} - \proj{\HSig}{\tbfV_k}\|_2 = \left\|\frac{\inner{\tbfV_k}{\bfD\bfD^T}_F}{\|\bfD\bfD^T\|_F^2} - \frac{\inner{\tbfV_k}{\HSig}_F}{\|\HSig\|_F^2}\right\|_2 \times \|\tbfV_k\|_2
\]
\begin{equation}\label{eqn:oracle_term_three}
=\left|\inner{\tbfV_k}{\Sigma_{\bfD} - \frac{\|\bfD\bfD^T\|_F^2}{\|\HSig\|_F^2}\HSig}_F\right|\times \frac{\|\tbfV_k\|_2}{\|\bfD\bfD^T\|^2_F}\leq \left\|\bfD\bfD^T - \frac{\|\bfD\bfD^T\|_F^2}{\|\HSig\|_F^2 }\HSig\right\|_2 \times \frac{\|\tbfV_k\|_F\|\tbfV_k\|_2}{\|\bfD\bfD^T\|_F^2}
\end{equation}
By Lemma \ref{lemma:oracle_emp}, as long as $N$ is large enough we can substitute $\|E\tbfV_k\|_2$ and $\|E\tbfV_k\|_F$ for $\|\tbfV_k\|_2$ and $\|\tbfV_k\|_F$ respectively up to a constant factor. From \ref{lemma:oracle_cov_exp}, we have
\[
\|E\tbfV_k\|_F \leq \|\bfd_k\bfd_k^T\|_F + \frac{Cs}{K}\|\bfD\bfD^T\|_F \leq \frac{Cs}{\sqrt{M}}
\]
and $\|E\tbfV_k\|_2 \leq C$ by a similar computation. Thus,
\[
\frac{\|\tbfV_k\|_F\|\tbfV_k\|_2}{\|\bfD\bfD^T\|_F^2} \leq C\sqrt{M} \times \frac{M}{K^2} \leq \frac{M^{3/2}}{K^2}
\]
In the proof of Lemma \ref{lemma:empCnvgc}, we showed that
\[
\left\|\bfD\bfD^T - \frac{\|\bfD\bfD^T\|_F^2}{\|\HSig\|_F^2 }\HSig\right\|_2 \leq \frac{CK\sqrt{\alpha\log M}}{\sqrt{M}\sqrt{N}}.
\]
with probability at least $1-2M^{-\alpha}$. Thus we can plug back into \ref{eqn:oracle_term_three} to yield
\[
\|\proj{\bfD\bfD^T}{\tbfV_k} - \proj{\HSig}{\tbfV_k}\|_2 \leq \frac{CK\sqrt{\alpha\log M}}{\sqrt{M}\sqrt{N}} \times \frac{CM^{3/2}}{K^2} \leq \frac{CM\sqrt{\alpha\log M}}{K\sqrt{N}}.
\]
Combining the pieces, we conclude that the desired bound holds with probability at least $1-K\exp(-Ns/(10K)-(K+2)M^{-\alpha} \geq 1-K\exp(-Ns/(10K)-2KM^{-\alpha}$. The bound on $\|\tbfd_k - \bfd_k\|_2$ follows immediately from Weyl's theorem and the Davis-Kahan theorem.
\end{proof}

\vskip 0.2in
\bibliography{_refs}

\end{document}